\documentclass[12pt]{article}
\usepackage{amsmath,amssymb,amsfonts}
\usepackage{amsthm}
\usepackage{graphicx}

\usepackage{booktabs}      % For nicer tables
\usepackage{geometry}      % For easy margin control
\usepackage{mathtools}     % For \coloneqq etc.
\usepackage{multirow}
\usepackage[authoryear,round]{natbib}
\usepackage[colorlinks=true,linkcolor=blue,citecolor=blue,urlcolor=blue]{hyperref}  
\usepackage{float}
\usepackage{bookmark}

\DeclarePairedDelimiter{\norm}{\lVert}{\rVert}

\theoremstyle{plain}%\theoremstyle{thmstyleone}%

\newtheorem{theorem}{Theorem}
\newtheorem{proposition}[theorem]{Proposition}% 
\newtheorem{definition}[theorem]{Definition}
\newtheorem{lemma}[theorem]{Lemma}
\newtheorem{remark}[theorem]{Remark}
\newtheorem{example}[theorem]{Example}
\title{Robust inference using density-powered Stein operators}
\author{Shinto Eguchi\thanks{The Institute of Statistical Mathematics, Japan  \;  ( {\tt eguchi@ism.ac.jp} )}}
\date{}

\begin{document}

\maketitle

\begin{abstract}
 {We introduce a density-power weighted variant of the Stein operator, called the $\gamma$-Stein operator, for robust inference with unnormalized probability models. The operator is motivated by the first variation of the $\gamma$-divergence under infinitesimal escort transport and weights the usual Stein field by a positive power of the model density. This weighting down-weights observations in low model-density regions, providing a principled robustness mechanism while retaining the normalizing-constant-free structure of score matching. We develop the resulting $\gamma$-score matching estimating equations and discuss their non-integrable, generalized-method-of-moments character. We further study two extensions: a $\gamma$-kernelized Stein discrepancy, interpreted as a robust diagnostic or contaminated-null goodness-of-fit procedure, and $\gamma$-Stein variational gradient descent for robust posterior approximation. Numerical examples on directional, mixture, and quartic-potential models illustrate the robustness--efficiency trade-off: positive $\gamma$ can stabilize inference under targeted contamination, whereas $\gamma=0$ remains preferable under clean well-specified models.}

\end{abstract}

\bigskip
\noindent
{\bf keywords}: {$\gamma$-divergence, goodness-of-fit, score matching, Stein discrepancy, Stein variational gradient descent}

\setcitestyle{authoryear,open={(},close={)}}

\newpage
\section{Introduction}

The theory of optimal transport has attracted broad attention not only in mathematics but also in statistics, machine learning, and artificial intelligence \citep{Villani2003, PeyreCuturi2019}. 
Recently, the concept of transport geometry has been explored within the framework of information geometry, aiming to incorporate data-space geometry via spatial gradient, Laplacian, and related differential operators  \citep{li2018information, mallasto2020entropy, ay2024information, cheng2023some}. 
 {These approaches incorporate data-space geometry into divergence-based methods
through spatial gradients, Laplacians, and related differential operators.}

Parameter estimation is a central task in statistics and machine learning. The most common method, 
 {MLE is statistically efficient under correct specification, but it can be
sensitive to outliers and data contamination.}
Score matching is efficient for a situation where the normalizing constant is intractable, or prohibitively expensive to compute \citep{hyvarinen2005estimation, lyu2012interpretation}.

We address this challenge by developing a robust estimation framework grounded in transport-based information geometry.
 {We introduce a probability-weighted operator, called the $\gamma$-Stein operator. It satisfies a weighted Stein identity and is also connected to the first variation of the $\gamma$-divergence under infinitesimal transport. These two roles should be distinguished: the Stein identity supplies unbiased estimating equations, whereas the variational identity explains why the weighting is naturally associated with the scale-invariant $\gamma$-divergence.}
 {The resulting $\gamma$-score matching estimating equation inherits the normalizing-constant-free property of ordinary score matching because the unknown partition function appears only as a common multiplicative factor. At the same time, the factor $p_\theta(x)^\gamma$ down-weights observations assigned small model density, giving robustness against gross contamination at the possible cost of efficiency under clean data.}
 {We demonstrate this robustness--efficiency trade-off through numerical experiments rather than claiming uniform dominance over likelihood-based competitors. We then extend the framework to a Reproducing Kernel Hilbert Space (RKHS), in which the proposed Stein operator induces a probability-weighted Stein discrepancy. This extension yields a robust diagnostic version of kernelized goodness-of-fit testing and a robust particle-based variational inference algorithm.}
Our work builds upon several lines of research in divergence measures and their applications in statistical inference. Classical Fisher divergence has been studied extensively for parameter estimation and model assessment \citep{li2020fisher, anastasiou2023stein}.
Extensions like score matching and Stein discrepancies 
\citep{liu2016stein, gorham2017measuring, matsubara2022robust} have provided tools for likelihood-free inference and robust learning.
Our approach is built upon the $\gamma$-divergence, a family of information-theoretic measures, such as $\beta$-divergence and $\gamma$-divergence, have been proposed to improve robustness against model misspecification \citep{basu1998robust, fujisawa2008robust, cichocki2010families}. 
{\color{black} A density-weighted Stein operator of the type studied here has also been
employed in generative modeling, where it induces an implicit geometric
regularization in flow matching \citep{eguchi2026densityweighted}.  The present
paper is concerned instead with statistical inference for unnormalized
models, namely parameter estimation, goodness-of-fit diagnostics, and
particle-based posterior approximation.}

The paper is organized as follows:  Section 2 introduces the theoretical foundation of our method, deriving the $\gamma$-Stein operator.
 {Section 2 also separates two related but different facts: the weighted Stein identity used for estimation and the escort-transport variation identity for the $\gamma$-divergence.}
Section 3 presents the resulting $\gamma$-score matching estimator, discusses statistical properties, and presents illustrative examples.
In Section 4 we extend this to construct our robust goodness-of-fit diagnostic and variational inference algorithm.  Finally, Section 5 discusses the implications of our findings and outlines future research directions.

\section{Stein operators}
We first review the standard Stein operator, showing how its fundamental zero-expectation identity provides a direct link to the Stein discrepancy and the Fisher divergence.  
We begin with general notation and admissibility conditions, and then state the operator definition and its key properties.
The score matching estimation method is naturally introduced based on the foundation of these notions. 
Secondly, our proposed $\gamma$-Stein operator as a weighted generalization designed for robustness is introduced. 
This establishes the new operator's fundamental connection to the $\gamma$-divergence and demonstrates how it provides a principled, robust estimation framework that, like standard score matching, remains independent of intractable normalizing constants. 

\subsection{Standard Stein operator}\label{sub1}

In statistics and machine learning, a fundamental challenge is to determine how `close' two probability distributions are. Stein's method provides a framework for this task. The central idea is to define a special mathematical tool, called Stein operator, that is uniquely associated with a specific probability distribution. This operator allows us to not only measure the difference between distributions but also to build effective statistical methods.

Before proceeding, we fix some notation and basic regularity.
Throughout, $p$ and $q$ denote smooth, strictly positive densities on $\mathbb{R}^{d}$ (possibly unnormalized) with gradients $s_{p}(x)=\nabla_x \log p(x)$ and $s_{q}(x)=\nabla_x\log q(x)$, where $\nabla_x$ denotes the gradient with respect to $x$.
Test functions $f:\mathbb{R}^{d}\to\mathbb{R}^{d}$ are assumed to be sufficiently smooth and integrable so that all derivatives and integration-by-parts identities below are valid and no boundary terms arise.

\begin{definition}[Stein operator]
The Stein operator $\mathcal A_p$, associated with the density $p$, acts on a suitable vector field $f(x)$ and is defined as:
\[
\mathcal{A}_p f(x) = \langle s_p(x), f(x)\rangle + \nabla_x \cdot f(x),
\]
where $\langle \;,\;\rangle$ is the Euclidean inner product, and $\nabla_x\cdot f(x)$ is the (geometric) divergence of $f$.
\end{definition}
The classical Stein operator satisfies the Stein identity: its expected value is zero if and only if the expectation is taken with respect to its own density $p$.
\begin{align}\nonumber
\mathbb{E}_{X \sim p}[\mathcal{A}_p f(X)] = 0.
\end{align} 
This identity holds for any sufficiently smooth function $f$ that vanishes at the boundaries of the support of $p$. It essentially means the operator $\mathcal{A}_p$ acts as a detector for the density $p$.
We can leverage this property to measure the difference between $p$ and another density $q$. If we apply the operator for $q$, $\mathcal{A}_q$, but take the expectation under $p$, the result will generally be non-zero unless $p=q$. This gives rise to the Stein discrepancy.
The Stein discrepancy between densities $p$ and $q$ over a class of functions $\mathcal{F}$ is defined by:
\begin{align}\nonumber
S(p\|q |\, \mathcal{F}) = \sup_{f \in \mathcal{F}} \left( \mathbb{E}_{X \sim p}[\mathcal{A}_q f(X)] \right)^2.
\end{align} 
By choosing different function classes $\mathcal{F}$, we can generate different types of (information) divergences.

One example is the Fisher divergence, defined as the expected squared norm of the difference between the score functions:
\begin{align}\nonumber
D_{\text{F}}(p\|q) = \mathbb{E}_{X \sim p}[\,\|s_p(X) - s_q(X)\|^2].
\end{align} 
The Fisher divergence is a special case of the Stein discrepancy when the function class $\mathcal{F}$ is the unit ball $\mathcal{B}$ in the space $L^2(p)$. To see this, we can rewrite the expectation term in the Stein discrepancy. Using integration by parts (which is the source of the Stein identity), we find:
\begin{align}\nonumber
 {\mathbb{E}_{X \sim p}[\mathcal{A}_q f(X)] = \mathbb{E}_{X \sim p}[\langle s_q(X) - s_p(X), f(X)\rangle] = \langle s_q - s_p, f \rangle_{L^2(p)}.}
\end{align} 
 {By the Cauchy--Schwarz inequality, maximizing the absolute value of this inner product over all functions $f$ in the unit ball $\mathcal{B}$ gives precisely the $L^2(p)$ norm of $s_q-s_p$, and squaring it yields the Fisher divergence. Thus, $S(p\| q | \mathcal{B}) = D_{\text{F}}(p\|q)$.}

This connection is the key to an elegant estimation technique called score matching \citep{hyvarinen2005estimation, vincent2011connection}. 
The goal of score matching is to fit a parametric model $p_\theta(x)$ to a data-generating density $p(x)$ by minimizing the Fisher divergence $D_{\text{F}}(p\| p_\theta)$.
A major challenge is that $D_{\mathrm{F}}$ depends on the unknown score $s_p$. However, thanks to the Stein identity, the objective function can be simplified to a form that only depends on the model's score $s_{\theta}=\nabla_x \log p_{\theta}(x)$ and the data:
\begin{align}\nonumber
D_{\text{F}}(p\| p_\theta) = \mathbb{E}_{X \sim p}[\|s_{\theta}(X)\|^2 + 2\nabla_x \cdot s_{\theta}(X)] + \text{const.}
\end{align} 
where the constant term depends only on $p$ and not on the parameter $\theta$. To estimate $\theta$ from a dataset $\{x_i\}_{i=1}^n$, we minimize the corresponding empirical objective function:
\begin{align}\nonumber
L_n(\theta) = \frac1n\sum_{i=1}^n \left\{ \|s_{\theta}(x_i)\|^2 + 2\nabla_x \cdot s_{\theta}(x_i) \right\}.
\end{align} 
The optimal parameter $\hat{\theta}$ is found by solving the estimating equation $\mathcal E_n(\theta)=0$, where
$\mathcal E_n(\theta)=\nabla_\theta L_n(\theta)$ is given by
\begin{align}\label{EstEqu}
\mathcal{E}_n(\theta)=\frac2n\sum_{i=1}^n \{\langle s_{\theta}(x_i),\nabla_\theta s_{\theta}(x_i)\rangle+\nabla_x\cdot \nabla_\theta s_{\theta}(x_i)\},
\end{align}
where $\nabla_\theta$ is the parameter gradient.
The validity of this approach is confirmed by Stein's method. 
The term inside the summation is proportional to the standard Stein operator $\mathcal{A}_{p_\theta}$ applied to the vector field $f(x) = \nabla_\theta s_{\theta}(x)$. 
The Stein identity yields  the expectation of this operator is zero under the model's own density, i.e., $\mathbb{E}_{X\sim p_{\theta}}[\mathcal{A}_{p_\theta}f(X)]=0$. 
This ensures that the estimating function is unbiased at the true parameter value, making it a sound basis for estimation.

A useful feature of score matching is that it is free from the normalizing constant problem.
The score $s_{\theta}(x) = \nabla_x \log p_\theta(x)$ involves the log of the density. If $p_\theta(x) = \frac{1}{Z_\theta}\tilde{p}_\theta(x)$, the normalizing constant $Z_\theta$ becomes an additive term $\log Z_\theta$ inside the logarithm. Since the gradient $\nabla_x$ is taken with respect to $x$, this term vanishes, allowing us to perform the entire estimation without ever needing to compute the often-intractable $Z_\theta$.

\subsection{\texorpdfstring{$\gamma$}{gamma}-Stein operator}\label{sub22}

We introduce a weighted version of the Stein operator to build a more robust estimation framework. 
This operator is designed to systematically down-weight the influence of outlier data points, a property achieved by introducing a power-law weight based on the probability density function itself.

\begin{definition}[$\gamma$-Stein Operator]
For a density function $p(x)$ and a tuning parameter $\gamma\geq0$, the  {$\gamma$-Stein operator} for a vector field $f$ is defined by
\[
\mathcal{A}_{p}^{(\gamma)}f(x) \coloneqq  p(x)^\gamma \left\{(\gamma+1) \langle s_{p}(x), f(x)\rangle  
+  \nabla_x\cdot f(x) \right\}.
\]
\end{definition}
When $\gamma=0$, this operator reduces to the classical Stein operator, ${\mathcal A}_p^{(0)}={\mathcal A}_p$.
Importantly, it satisfies a corresponding $\gamma$-Stein identity for an admissible test function $f$, as integration by parts shows that its expectation under $p$ vanishes:
\begin{align}\nonumber
\mathbb{E}_{X\!\sim p}[\mathcal{A}_p^{(\gamma)} f(X)] = \int \{(\gamma+1)p^{\gamma+1}\langle s_p,  f\rangle -
\langle\nabla_x p^{\gamma+1},f\rangle \}dx =0.
\end{align}
We work under routine smoothness and `no edge-effects' conditions so that the calculus steps used below are valid; these are satisfied by the models and kernels considered in our examples.
The $\gamma$-Stein operator provides a powerful dynamical perspective on particle-based inference, extending the intuition from the standard Stein operator. 
 It is worthwhile to note an alternative expression for the $\gamma$-Stein operator:
\[
\mathcal{A}_p^{(\gamma)}f(x) = \frac{\nabla \cdot (p(x)^{\gamma+1} f(x))}{p(x)}.
\]
This definition clarifies the operator's structure as a normalized divergence. Its utility is immediately apparent in proving the $\gamma$-Stein identity $\mathbb{E}_{p}[\mathcal{A}_{p}^{(\gamma)}f] = 0$, as the expectation simplifies to $\int \nabla \cdot (p^{\gamma+1} f) dx$, which is zero by integration by parts under the assumed boundary conditions

 We denote by $\mathcal{F}_p^{(\gamma)}$ the collection of all test functions $f$ for which the weighted Stein operator $\mathcal{A}_p^{(\gamma)}$ is well-defined and all integration-by-parts identities hold without boundary contributions.
This class enlarges the standard Stein class: when $\gamma=0$ it reduces to the usual Stein set $\mathcal{F}_p^{(0)}$, and for $\gamma>0$ we have the inclusion $\mathcal{F}_p^{(0)}\subseteq \mathcal{F}_p^{(\gamma)}$.
 {Here the inclusion is meant under the standing tail and boundary assumptions used throughout the paper. If $f\in\mathcal{F}_p^{(0)}$, then $p f$ has no boundary contribution. Multiplication by $p^\gamma$ gives the boundary term $p^{\gamma+1}f$, which decays at least as fast as $p f$ for the positive smooth densities considered in our examples. Thus the weighted integration-by-parts identity remains valid. For distributions with compact support or nonstandard boundary behavior, this inclusion should be checked as part of the admissibility condition.}
To understand this, we first consider the two components of the operator it is built upon:
\begin{itemize}
\item An optimization term, $\langle s_p, f \rangle $, which directs particles to move ``uphill" on the log-probability surface towards regions of higher density.
\item A repulsive term, $\nabla_x \cdot f(x)$, which contributes to particle interactions. It acts as a repulsive force that encourages the ensemble of particles to spread out and cover the full distribution, preventing a collapse to a single mode.
\end{itemize}
Our $\gamma$-Stein operator modulates the influence of these two forces with the weighting factor $p(x)^\gamma$. This creates an adaptive dynamic:
\begin{itemize}
\item In high-density regions (where $p(x)$ is large), the $p(x)^\gamma$ factor amplifies the effect of the operator. Both the uphill force and the repulsive force are strong, so that particles tend to be more strongly affected by the operator in regions
where the target density is large.
\item In low-density regions (where $p(x)$ is small, such as the location of an outlier), the $p(x)^\gamma$ factor \emph{suppresses} the entire operator. Consequently, an outlier particle exerts a much weaker repulsive force on other 'good' particles and experiences a weaker pull towards the modes.
\end{itemize}
 {This suppression is one source of robustness, since an outlying particle has a
reduced influence on the resulting update.}
The following theorem provides a key insight, establishing a structure parallel to the classical case.

\begin{theorem}%[The $\gamma$-Stein Discrepancy] 
\label{Stein}
Let $\mu_{\gamma}(dx) \coloneqq p(x)q(x)^{\gamma}dx$ be a mixed weighting measure. The expectation of the $\gamma$-Stein operator under $p$ is the inner product of the score difference with $f$:
\[
\mathbb{E}_{p}\left[\mathcal{A}_{q}^{(\gamma)}f\right] = \left\langle s_{q}-s_{p}, f\right\rangle_{L^{2}(\mu_\gamma)}.
\]
\end{theorem}
\begin{proof}
By definition, the expectation is 
\begin{align}\nonumber
\mathbb{E}_{X\!\sim p}[\mathcal{A}_{q}^{(\gamma)}f(X)] = \int p(x) q(x)^\gamma \{(\gamma+1)\langle s_q, f\rangle + \nabla_x \cdot f\} dx.
\end{align}  
Using integration by parts on the second term gives:
    \begin{align}\nonumber
\int p q^\gamma \nabla_x \cdot f \,dx = - \int \langle \nabla_x(p q^\gamma), f \rangle dx.
\end{align} 
    We can expand the gradient term 
$\nabla_x(p q^\gamma) = q^\gamma (\nabla_x p) + p (\gamma q^{\gamma-1} \nabla_x q) = q^\gamma p s_p + \gamma p q^\gamma s_q$. Substituting this back leads to:
    \begin{align}\nonumber
\mathbb{E}_{X\sim p}[\mathcal{A}_{q}^{(\gamma)}f(X)] = \int p q^\gamma \{(\gamma+1)\langle s_q ,  f\rangle  - \langle s_p ,  f \rangle - \gamma \langle s_q,  f\rangle  \} dx = \int p q^\gamma \langle s_q - s_p, f \rangle dx
\end{align} 
    which gives the desired result.

\end{proof}
We define the $\gamma$-Stein discrepancy:
\[
S^{(\gamma)}(p,q;{\cal F})=
\sup_{f\in \mathcal{F}} (\mathbb{E}_{X\!\sim p}[\mathcal{A}_q^{(\gamma)} f(X)])^2,
\]
analogous to the Stein discrepancy. 
This theorem naturally suggests a weighted variant of Fisher divergence.   
Indeed, by taking  the  unit ball $\mathcal{B}_\gamma$ in $L^2(\mu_\gamma)$ in place of $\mathcal F$, the $\gamma$-Stein discrepancy reduces to 
\begin{align}\label{emp}
D_{\mathrm F}^{(\gamma)}(p\|q)= 
\mathbb E_{p}[q^{\gamma}\norm{s_{q}-s_{p}}^{2}].
\end{align}
This may be an elegant extension of the Fisher divergence, but is unsuitable for the statistical application for the score matching.
If we consider fitting a model $q=p_\theta$ to data from $p$, the objective function contains the term 
$\mathbb E_p[p_\theta^\gamma \|s_p\|^2]$.
This term depends on $s_p=\nabla_x \log p$, the score of the true (and unknown) data-generating distribution. Since this term cannot be calculated from data samples alone, we cannot construct a simple empirical objective function.
 {This observation explains why the weighted Fisher divergence in \eqref{emp} is not directly used as the empirical score-matching objective. The practical estimator in Section~\ref{sub31} will instead be built from the zero-expectation $\gamma$-Stein identity. Before doing so, we clarify the information-geometric origin of the same operator by studying the first variation of the $\gamma$-divergence under an infinitesimal transport of the model density.}

Let $v\colon\mathbb{R}^d\to\mathbb{R}^d$ be a vector field with suitable regularity and $\int \!\|v\|^2 q < \infty$. 
Define the infinitesimal transport map $T_{\varepsilon}(x)=x+\varepsilon v(x)$ and denote by $q_{\varepsilon}=T_{\varepsilon}{}_{\#}q$ its push-forward: 
\begin{align}\nonumber
q_{\varepsilon}(x)=q\bigl(x-\varepsilon v(x)\bigr)\,\det\bigl(\mathrm I_d-\varepsilon\nabla_x v(x)\bigr).
\end{align} 
up to the first-order of $\varepsilon$.
Consider the KL-divergence, $D_{\mathrm{KL}}(p\|q).$ 
Its G\^{a}teaux  derivative along this transport path is:
{ 
\begin{align}\nonumber
 \frac{d}{d\varepsilon} D_{\mathrm{KL}}(p\|q_\varepsilon)\Big|_{\varepsilon=0} =& \langle s_q-s_p,v \rangle_{L^2(p)}\\[3pt]\nonumber
=& \mathbb E_{X\!\sim p}[{\mathcal A}_q v(X)].
\end{align}
}
 {This follows from $\dot q_{\varepsilon}(x)|_{\varepsilon=0} = -\nabla_x\cdot(q v)(x)$. In particular, when $q=p$, the right-hand side becomes $\mathbb E_p[\mathcal A_p v]=0$, which is precisely the ordinary Stein identity.}

 {We now show how the same operator arises from the logarithmic
$\gamma$-divergence of \citet{fujisawa2008robust}.  We use the scale-invariant form}
{ 
\begin{align}\label{eq:log-gamma-div}
D_\gamma(p\|q)
=\frac{1}{\gamma(\gamma+1)}\log\int p^{\gamma+1}dx
 +\frac{1}{\gamma+1}\log\int q^{\gamma+1}dx
 -\frac{1}{\gamma}\log\int p q^\gamma dx .
\end{align}
}
{  This logarithmic form is invariant under $q\mapsto cq$ and converges to the
Kullback--Leibler divergence as $\gamma\downarrow0$.  Define}
{ 
\begin{align}\label{escorts}
I_\gamma(p,q)=\int p(x)q(x)^\gamma dx,
\qquad
J_{\gamma+1}(q)=\int q(x)^{\gamma+1}dx.
\end{align}}
{ The logarithmic
$\gamma$-divergence is naturally defined on the projective density class
$[q]=\{cq:c>0\}$.  The push-forward is compatible with this quotient since
$T_{\varepsilon\#}(cq)=cT_{\varepsilon\#}q$.}

 {To obtain the \(\gamma\)-Stein operator itself directly as a first variation,
we transport the \((\gamma+1)\)-escort density \(q^{\gamma+1}\) rather than
\(q\) itself.}

\begin{definition}[{ {Escort transport}}]\label{def:escort-transport}
{  Let $T_\varepsilon(x)=x+\varepsilon v(x)$.  The escort transport of $q$ along
$v$ is the path $q_\varepsilon^{\mathrm{esc}}$ defined by}
\begin{align}\label{eq:escort-transport}
{ \bigl(q_\varepsilon^{\mathrm{esc}}\bigr)^{\gamma+1}
=T_{\varepsilon\#}\bigl(q^{\gamma+1}\bigr).}
\end{align}
\end{definition}
{  The continuity equation for $q^{\gamma+1}$ gives}
\begin{align}\label{eq:escort-generator}
{ 
\dot q_0^{\mathrm{esc}}(x)
:=\left.\frac{d}{d\varepsilon}q_\varepsilon^{\mathrm{esc}}(x)\right|_{\varepsilon=0}
=-\frac{1}{(\gamma+1)q(x)^\gamma}
\nabla\!\cdot\!\bigl(q(x)^{\gamma+1}v(x)\bigr).}
\end{align}
 {Since
$\mathcal A_q^{(\gamma)}v=\nabla\!\cdot(q^{\gamma+1}v)/q$, this is equivalently}
{ 
\begin{align}\label{eq:escort-generator-Astein}
\dot q_0^{\mathrm{esc}}(x)
=-\frac{q(x)^{1-\gamma}}{\gamma+1}\,
\mathcal A_q^{(\gamma)}v(x).
\end{align}
}
 {Thus the escort transport is the coordinate system in which the
$\gamma$-Stein operator directly generates the variation.}

\begin{proposition}[ {Escort transport and the logarithmic $\gamma$-divergence}]
\label{gamma-div}
 {Let $q_\varepsilon^{\mathrm{esc}}$ be defined by
\eqref{eq:escort-transport}.  Then}

\begin{align}\label{eq:log-gamma-variation}
{  \left.\frac{d}{d\varepsilon}
D_\gamma\bigl(p\|q_\varepsilon^{\mathrm{esc}}\bigr)
\right|_{\varepsilon=0}
=
\frac{1}
{\gamma+1}
\frac{\mathbb E_p[\mathcal A_q^{(\gamma)}v]}
{I_\gamma(p,q)}.}
\end{align}
\end{proposition}

\begin{proof}
 {By the definition of escort transport,}
{ 
\begin{align*}
\left.\frac{d}{d\varepsilon}
\bigl(q_\varepsilon^{\mathrm{esc}}\bigr)^{\gamma+1}
\right|_{\varepsilon=0}
=-\nabla\!\cdot\!\bigl(q^{\gamma+1}v\bigr).
\end{align*}
}
 {This proves \eqref{eq:escort-generator} and
\eqref{eq:escort-generator-Astein}.  In $D_\gamma(p\|q)$, the first term is
independent of $q$, while the second term is constant along the escort
transport because}
{ 
\begin{align*}
\int \bigl(q_\varepsilon^{\mathrm{esc}}\bigr)^{\gamma+1}dx
=\int q^{\gamma+1}dx .
\end{align*}
}
 {Thus only $I_\gamma(p,q)=\int p q^\gamma dx$ contributes.  Using
\eqref{eq:escort-generator-Astein},}
{ 
\begin{align*}
\left.\frac{d}{d\varepsilon} I_\gamma(p,q_\varepsilon^{\mathrm{esc}})
\right|_{\varepsilon=0}
&=\gamma\int p q^{\gamma-1}\dot q_0^{\mathrm{esc}}dx  \\
&=-\frac{\gamma}{\gamma+1}
\int p\,\mathcal A_q^{(\gamma)}v\,dx.
\end{align*}
}
 {Substitution into the derivative of
$-(1/\gamma)\log I_\gamma(p,q_\varepsilon^{\mathrm{esc}})$ gives
\eqref{eq:log-gamma-variation}.}
\end{proof}

 {For comparison, we also record the corresponding calculation for the
density-power, or $\beta$-, divergence.  Although the notation is parallel,
$D_\beta$ below is not the same functional as the scale-invariant
$D_\gamma$ above after merely setting $\beta=\gamma$.}
Similarly, the $\beta$-divergence \citep{basu1998robust, mihoko2002robust, cichocki2010families},
\[
D_\beta(p\|q)=\frac{1}{\beta(\beta+1)}\int p^{\beta+1}dx+\frac{1}{\beta+1}\int q^{\beta+1}dx-\frac{1}{\beta}\int p q^{\beta}dx,
\]
essentially suggests the first variation is equal to the expected
$\beta$-Stein operator.

{ 
\begin{proposition}[First variation of the \(\beta\)-divergence under escort transport]
Let \(q_\varepsilon^{\mathrm{esc}}\) be the escort transport of \(q\) along \(v\),
defined by
\[
\bigl(q_\varepsilon^{\mathrm{esc}}\bigr)^{\beta+1}
=
T_{\varepsilon\#}\bigl(q^{\beta+1}\bigr).
\]
Then
\begin{align}
\left.
\frac{d}{d\varepsilon}
D_\beta\!\bigl(p\|q_\varepsilon^{\mathrm{esc}}\bigr)
\right|_{\varepsilon=0}
=
\frac{1}{\beta+1}
\mathbb E_p\!\left[\mathcal A_q^{(\beta)}v\right].
\label{eq:beta-escort-variation}
\end{align}
\end{proposition}

\begin{proof}
The escort transport gives
\[
\dot q_0^{\mathrm{esc}}
=
-\frac{1}{(\beta+1)q^\beta}
\nabla\cdot(q^{\beta+1}v)
=
-\frac{q^{1-\beta}}{\beta+1}
\mathcal A_q^{(\beta)}v .
\]
Since \(\int (q_\varepsilon^{\rm esc})^{\beta+1}dx\) is preserved, only
\(-\beta^{-1}\int p(q_\varepsilon^{\rm esc})^\beta dx\) contributes. Hence
\[
\left.
\frac{d}{d\varepsilon}D_\beta(p\|q_\varepsilon^{\rm esc})
\right|_{\varepsilon=0}
=
-\int p q^{\beta-1}\dot q_0^{\rm esc}dx
=
\frac{1}{\beta+1}
E_p\!\left[A_q^{(\beta)}v\right].
\]
This completes the proof.
\end{proof}

The escort-transport calculation gives a parallel variational representation
for the density-power divergence.  The main difference from the logarithmic
\(\gamma\)-divergence is scale invariance: \(D_\gamma(p\|q)\) is projective in
\(q\), whereas \(D_\beta(p\|q)\) is not.  Hence both divergences lead to
density-powered Stein operators locally, but the logarithmic \(\gamma\)-form is
more naturally aligned with unnormalized models.
}

\begin{remark}[Independence from normalizing constant]\label{intractable}
A key advantage of the $\gamma$-Stein framework is its applicability to unnormalized models. Let a parametric model be specified by its unnormalized density $u_{\theta}(x)$, such that the full probability density is $p_{\theta}(x) = u_{\theta}(x) / Z_{\theta}$, where $Z_{\theta} = \int u_{\theta}(x) dx$ is the often intractable normalizing constant.

The framework's utility rests on two properties.
\begin{enumerate}
\item The score function is independent of $Z_{\theta}$:
    $$s_{\theta}(x) = \nabla_{x} \log p_{\theta}(x) = \nabla_{x} \log(u_{\theta}(x)/Z_{\theta}) = \nabla_{x} \log u_{\theta}(x).$$
This is the well-known property that standard score matching relies on.

\item The normalizing constant cancels from the estimating equation: 
While the score is independent of $Z_{\theta}$, the $\gamma$-Stein operator itself is not, due to the weighting term:
    $$
\mathcal{A}_{p_{\theta}}^{(\gamma)}f(x) = p_{\theta}(x)^{\gamma} \{(\gamma+1)\langle s_{\theta}(x), f(x)\rangle +\nabla_{x}\cdot f(x)\}.
$$
    Here, the weight $p_{\theta}(x)^{\gamma}$ becomes $(u_{\theta}(x)/Z_{\theta})^{\gamma} = u_{\theta}(x)^{\gamma} / Z_{\theta}^{\gamma}$. However, any estimating equation formed by setting the empirical average to zero, $\frac{1}{n}\sum_{i=1}^{n} \mathcal{A}_{p_{\theta}}^{(\gamma)}f(x_i) = 0$, takes the form:
    $$
\frac{1}{n}\sum_{i=1}^{n} \frac{u_{\theta}(x_i)^{\gamma}}{Z_{\theta}^{\gamma}}
 {\{(\gamma+1)\langle \nabla_x \log u_{\theta}(x_i), f(x_i)\rangle +\nabla_{x}\cdot f(x_i)\} = 0.}
$$
    Because $Z_{\theta}^{\gamma}$ is constant in $x_i$, it cancels from the estimating equation, leaving an equation that depends only on $u_{\theta}(x)$.
\end{enumerate}
This cancellation ensures that the entire estimation procedure is free from the need to compute $Z_{\theta}$, preserving the computational advantage of score matching while adding the robustness properties of the $\gamma$-weighting.

\end{remark}

We can generalize the $\gamma$-Stein operator as follows.   By a fixed function  ${w}:[0,\infty)\to[0,\infty) $, the generalized Stein operator is defined as
\[
         {\mathcal A}_q^{({w} )} f =\{{w} (q)+q{w} ^\prime(q)\}\langle s_q,f\rangle + {w} (q)\nabla\cdot f
\]
for a vector field $f$.  If ${w} (q)=q^\gamma$, then the generalized Stein operator reduces to the $\gamma$-Stein operator:
$\mathcal A_q^{({w} )}=\mathcal A_q^{(\gamma)}$.
An argument similar to that in the proof of Theorem \ref{Stein} yields
\[
\mathbb{E}_{p}\left[\mathcal{A}_{q}^{({w} )}f\right] = \left\langle s_{q}-s_{p}, f\right\rangle_{L^{2}(\mu_{w} )},
\]
where $L^{2}(\mu_{w} )$ denotes the $L^2$-space with respect to a measure $\mu_{w} (dx)=p(x){w} (q(x))dx$.
This implies the Stein identity: 
\[
\mathbb{E}_{p}\left[\mathcal{A}_{p}^{({w} )}f\right] =0
\]
for all $f$.
We can consider the generalized Stein discrepancy,
\(
S^{({w} )}(p,q;{\cal F})=
\sup_{f\in \mathcal{F}} (\mathbb{E}_{X\!\sim p}[\mathcal{A}_q^{({w} )} f(X)])^2.
\)
Similarly, by taking  the  unit ball $\mathcal B_{w}$ in $L^2(\mu_{w} )$ as $\mathcal F$, 
 the generalized Stein discrepancy is reduced to a form of weighted Fisher divergence, 
\begin{align}\nonumber
D_{\mathrm F}^{({w} )}(p\|q)= 
\mathbb E_{p}[{w} (q)\norm{s_{q}-s_{p}}^{2}].
\end{align}
If we consider the whole framework to keep working with unnormalized models $q=u/Z$, 
we need the operator (and discrepancy) to be invariant to scaling $q\mapsto cq$. 
%This holds if and only if ${w} $ is homogeneous: ${w} (cq)=c^\gamma {w}  (q)$, which forces ${w} (q)=C q^\gamma$
The ${w} $-family is valuable conceptually, but for unnormalized targets the only scale-invariant choices are essentially the 
$\gamma$-family up to a constant.
The characterization is given by the following proposition.

\begin{proposition}
Let $w:(0,\infty)\to(0,\infty)$ be Borel measurable. For an unnormalized density $q:\Omega\to(0,\infty)$ define the $w$-weighted expectation
\[
E_{q,w}[h]\;=\;\frac{\int_\Omega w(q(x))\,q(x)\,h(x)\,dx}{\int_\Omega w(q(x))\,q(x)\,dx}\,.
\]
Call $w$ \emph{scale-invariant} if for every $c>0$ and every integrable $h$,
\[
E_{c q,w}[h]\;=\;E_{q,w}[h]\,.
\]
Then the following are equivalent:
\begin{itemize}
\item[] {\rm (i)}. $w$ is scale-invariant.

\item[] {\rm (ii)}. There exists $\gamma\in\mathbb{R}$ such that $w(c t)=c^{\gamma} w(t)$ for all $c,t>0$.
\end{itemize}
In particular, $w(t)=K\,t^{\gamma}$ for some $K>0$; i.e., up to a constant factor, $w$ is a power law.
\end{proposition}
\begin{proof}
{ 
(i)$\Rightarrow$(ii): Fix $c>0$. The identity $E_{c q,w}[h]=E_{q,w}[h]$ for all bounded measurable $h$ implies that the two finite measures
\[
\mu_{cq}(dx)=w(cq(x))\,c q(x)\,dx,
\qquad
\mu_q(dx)=w(q(x))\,q(x)\,dx
\]
have the same normalized version. Hence they are proportional: there exists $k(c,q)>0$ such that
\[
w(cq(x))\,c q(x)=k(c,q)w(q(x))q(x)
\quad\text{a.e.}
\]
Equivalently,
\[
w(cq(x))=\alpha(c,q)w(q(x)),
\qquad \alpha(c,q)=k(c,q)/c.
\]
Since the assertion is required for every positive density $q$, we may choose $q$ taking arbitrary positive values on sets of positive measure. This forces the proportionality factor to depend only on $c$, not on the particular value $t=q(x)$; thus there exists $\alpha(c)>0$ such that
\[
w(ct)=\alpha(c)w(t),\qquad c,t>0.
\]
Applying this identity twice gives
\[
w(c_1c_2t)=\alpha(c_1)w(c_2t)=\alpha(c_1)\alpha(c_2)w(t),
\]
whereas applying it once with $c_1c_2$ gives
\[
w(c_1c_2t)=\alpha(c_1c_2)w(t).
\]
Because $w(t)>0$, it follows that $\alpha(c_1c_2)=\alpha(c_1)\alpha(c_2)$. A positive measurable multiplicative function on $(0,\infty)$ has the power form $\alpha(c)=c^\gamma$ for some $\gamma\in\mathbb{R}$. Hence $w(ct)=c^\gamma w(t)$. Taking $t=1$ gives $w(t)=K t^\gamma$ with $K=w(1)>0$.

(ii)$\Rightarrow$(i): If $w(ct)=c^{\gamma}w(t)$, then
\[
\frac{\int w(c q)\,c q\,h}{\int w(c q)\,c q}
=\frac{c^{\gamma+1}\int w(q)\,q\,h}{c^{\gamma+1}\int w(q)\,q}
=\frac{\int w(q)\,q\,h}{\int w(q)\,q}=E_{q,w}[h]\,.
\]
Thus scale-invariance holds.
} 
\end{proof}

\section{Score matching via \texorpdfstring{$\gamma$}{gamma}-Stein operator}\label{sec3}
We  formally define the $\gamma$-Score Matching Estimator ($\gamma$-SME) based on the $\gamma$-Stein identity, and establishes its key properties: independence from normalizing constants and its non-integrable nature (asymmetric Jacobian). We next demonstrate the method's practical utility and robustness by applying it to models with intractable normalizers, including distributions on the unit sphere (vMF, Fisher-Bingham), normal mixtures, and a quartic potential model. Finally we address the practical choice of the robustness parameter $\gamma$ by introducing a principled, robust cross-validation scheme for selecting $\gamma$.

\subsection{General properties and efficiency}\label{sub31}
 {Standard score matching, corresponding to $\gamma=0$ in our framework, is a
method for fitting models with intractable normalizing constants.  By minimizing
the Fisher divergence, it leads to an objective depending only on the model
score, and hence does not require evaluation of the normalizing constant.
This idea is also related to recent score-based generative modeling \citep{song2019generative, ho2020denoising}.}

 {Section~\ref{sub22} gave two complementary facts. The first 
is the $\gamma$-Stein identity 
$\mathbb{E}_{X\sim p}[\mathcal{A}_{p}^{(\gamma)}f(X)]=0$, 
which holds for admissible test functions and is the basis for estimation.  
The second is Proposition~\ref{gamma-div}, which expresses 
$\mathcal{A}_q^{(\gamma)}$ as the exact first variation of the logarithmic 
$\gamma$-divergence along the escort transport directly.  
Geometrically, the two facts together identify $\mathcal{A}_q^{(\gamma)}$ as 
the canonical Stein operator for the $(\gamma+1)$-escort coordinate of 
$D_\gamma$.  The estimating equation below uses the first fact, applied to 
the test function $f=\nabla_\theta s_\theta$.}

To define a specific estimator, we must choose a test function $f(x)$. A natural choice is the gradient of the model's score function with respect to its parameters, $f(x) = \nabla_\theta s_{\theta}(x)$, which captures how the model's structure changes with $\theta$. Applying the $\gamma$-Stein operator to this function gives our proposed $\gamma$-score matching estimating function:
\begin{align}\nonumber
U_\gamma(\theta,x) =& \mathcal{A}_{p_\theta}^{(\gamma)} \left( \nabla_\theta s_{\theta}(x) \right)
\\[3mm]\nonumber
 =& p_\theta(x)^\gamma \left\{ (\gamma+1) \langle s_{\theta}(x), \nabla_{\theta}s_{\theta}(x) \rangle + \nabla_x \cdot \nabla_{\theta}s_{\theta}(x) \right\},
\end{align}
where $\langle \cdot, \cdot \rangle$ denotes the inner product for each component of the parameter gradient.
For a dataset $\{x_i\}_{i=1}^n$, the $\gamma$-score matching estimator $\hat{\theta}_\gamma$ is 
defined by the value of $\theta$ that solves the estimating equation:
\begin{align}\nonumber
\bar{U}_\gamma(\theta) = \frac{1}{n} \sum_{i=1}^n U_\gamma(\theta, x_i) = 0.
\end{align}
By the $\gamma$-Stein identity, this estimating function is unbiased at the true parameter value, meaning $\mathbb{E}_{X \sim p_\theta}[U_\gamma(\theta,X)] = 0$, which makes it a sound basis for estimation.

A major advantage of this method is its independence from the often-intractable normalizing constant. 
Let the model be $p_\theta(x) = u_\theta(x) / Z_\theta$, where $u_\theta(x)$ is an unnormalized, easy-to-compute density.
\begin{itemize}
\item  The score function is of a tractable form: $s_{\theta}(x) = \nabla_x \log u_\theta(x)$.
\item  The weighting term $p_\theta(x)^\gamma$ becomes $u_\theta(x)^\gamma / Z_\theta^\gamma$.
\end{itemize}
The estimating equation can therefore be written as:
\begin{align}\nonumber
\frac{1}{n} \sum_{i=1}^n u_\theta(x_i)^\gamma \left\{ (\gamma+1) \langle s_{\theta}(x_i), \nabla_{\theta}s_{\theta}(x_i) \rangle + \nabla_x \cdot \nabla_{\theta}s_{\theta}(x_i) \right\} = 0.
\end{align} 
Since we are setting the equation to zero, the $1/Z_\theta^\gamma$ factor can be dropped, and the entire estimation can proceed without ever computing $Z_\theta$.
As discussed in Subsection \ref{sub22}, the choice of a power-law weight $w(p) \propto p^\gamma$ is unique in preserving this feature. For the cancellation to work, the weighting function must satisfy a differential equation whose only non-trivial solution is this power law.

An important property of the proposed estimator arises here. For $\gamma=0$, the estimating function $\bar U_0(\theta)$ is the gradient of the standard score matching objective function. In this case, its Jacobian matrix is symmetric (as it is a Hessian).
However, for $\gamma \ne 0$, the Jacobian matrix is generally \emph{ asymmetric}. This implies that the estimating function $\bar U_\gamma(\theta)$ is non-integrable, that is, there is no scalar objective function $\Phi_\gamma(\theta)$ such that $\bar U_\gamma(\theta) = \nabla_\theta \Phi_\gamma(\theta)$.
Nevertheless, the Jacobian matrix is asymptotically symmetric under correctly specified model $p_\theta$
on account of the following proposition:
\begin{proposition}\label{symmetry}
Let 
\begin{align}\nonumber
J_{\gamma}(\theta) = \mathbb E_{X\!\sim p_\theta}[\nabla_\theta \;U_\gamma (\theta,X)^\top]. 
\end{align}  
Then, $J_{\gamma}(\theta)$ is a symmetric matrix:
\begin{align}\label{sym}
J_{\gamma}(\theta) = \mathbb{E}_{X \sim p_\theta} \left[ p_\theta(X)^\gamma \left\langle \nabla_\theta s_{\theta}(X),\nabla_\theta^\top s_{\theta}(X) \right\rangle \right].
\end{align}

\end{proposition}
\begin{proof}
{ 
For clarity, write the $a$th component of the estimating function as
\[
U_{\gamma,a}(\theta,x)
=p_\theta(x)^\gamma\{(\gamma+1)\langle s_\theta(x),V_a(x)\rangle+\nabla_x\cdot V_a(x)\},
\qquad
V_a(x)=\partial_{\theta_a}s_\theta(x).
\]
Because $\mathbb E_{p_\theta}\{U_{\gamma,a}(\theta,X)\}=0$, differentiating this identity with respect to $\theta_b$ gives the Bartlett-type relation
\[
\mathbb E_{p_\theta}[\partial_{\theta_b}U_{\gamma,a}(\theta,X)]
=-\mathbb E_{p_\theta}[S_b(\theta,X)U_{\gamma,a}(\theta,X)],
\]
where $S_b(\theta,x)=\partial_{\theta_b}\log p_\theta(x)$. Since $\nabla_x S_b(\theta,x)=\partial_{\theta_b}s_\theta(x)=V_b(x)$, integration by parts gives
\begin{align*}
&\mathbb E_{p_\theta}\left[p_\theta^\gamma S_b\{(\gamma+1)\langle s_\theta,V_a\rangle+\nabla_x\cdot V_a\}\right]\\
&\quad =-\mathbb E_{p_\theta}\left[p_\theta^\gamma\langle V_b,V_a\rangle\right].
\end{align*}
Therefore
\[
J_{\gamma,ba}(\theta)
=\mathbb E_{p_\theta}[\partial_{\theta_b}U_{\gamma,a}(\theta,X)]
=\mathbb E_{p_\theta}\left[p_\theta(X)^\gamma\langle V_b(X),V_a(X)\rangle\right].
\]
In matrix form this is exactly \eqref{sym}. The right-hand side is a weighted Gram matrix with a positive sign, and is therefore symmetric.
}
\end{proof}
The Jacobian matrix $\nabla_\theta \bar U_\gamma^\top(\theta)$  almost surely converges to the symmetric matrix $J_\gamma(\theta)$, which implies asymptotic symmetry.   
The simplified expression \eqref{sym} leads that it can be viewed as an expected, weighted {Gramian matrix}.
Hence, the $\gamma$-score matching estimator asymptotically has a unique solution under the assumption such that the components
$\nabla_{\theta_j} s_{\theta}$ are functionally independent.
The large-sample behavior of the $\gamma$-score matching estimator is governed by its Godambe information (or "sandwich" covariance matrix), 
$$
\mathrm{Avar}(\theta) = J_\gamma(\theta)^{-1}V_\gamma(\theta)J_\gamma( \theta)^{-1}, 
$$
where  $V_\gamma(\theta)$ is the covariance matrix of $U_\gamma(\theta,X)$.

While the estimating function $U_\gamma(\theta)$ is in general asymmetric, it does not prevent consistent estimation. Instead, it points us toward the Generalized Method of Moments (GMM) as the natural framework for estimation. GMM is designed precisely for situations with a set of unbiased estimating equations that may not derive from a single objective function.
We can construct a GMM objective function by forming a quadratic form:
\begin{align}\nonumber
L_{\text{GMM}}(\theta) =  \bar U_\gamma(\theta)^\top W_n^{-1} \bar U_\gamma(\theta),
\end{align} 
where $W_n$ is a positive definite weighting matrix.
See \citet{hansen1982large} for the general discussion. Minimizing $L_{\text{GMM}}(\theta)$ yields a consistent and asymptotically normal estimator. 
It might be worth a brief mention that even the simplest choice, $W_n = {\mathrm I}$, yields a consistent estimator by minimizing the squared Euclidean norm $\|\bar U_\gamma(\theta)\|^2$. This provides a direct, practical objective function that generalizes the $\gamma=0$ case (which minimizes the norm of the gradient of the score matching objective).
This framework also provides a systematic way to improve efficiency by incorporating additional moment conditions  into an expanded estimating function, e.g.,
\[
     \bar U_\gamma^{(2)}(\theta)=\frac1n\sum_{i=1}^n \mathcal A_{p_\theta}^{(\gamma)}\begin{pmatrix}
               \nabla_\theta s_{\theta}(x_i)\\[2mm]
\nabla_\theta^{\otimes 2} s_{\theta}(x_i)
\end{pmatrix}.
\]
This principle of augmenting the estimating equations is not limited to second-order derivatives. In theory, one could include an entire family of test functions, leading to a much larger set of moment conditions. This raises the crucial question of optimal selection.
The GMM formalism provides a clear answer: The optimal GMM weighting matrix minimizes the size of the asymptotic covariance matrix of the estimator. However, a practical trade-off exists, and therefore the selection of a powerful yet parsimonious set of non-integrable estimating functions remains a key consideration for applying this framework.
However, we do not pursue this methodological discussion further as a complete treatment is beyond the scope of the present work.

Finally, we look at the $\gamma$-score matching estimator for one of the most basic models.
{\color{black}
Before proceeding, we fix the notation for the criterion associated with the
$\gamma$-divergence.  For $\gamma>0$, a model $p_\theta$, and a sample
$\{x_i\}_{i=1}^n$, define the scale-invariant ratio criterion
\begin{align}\label{gammaratio}
R_\gamma(\theta)
=\frac{\displaystyle \frac{1}{n}\sum_{i=1}^{n}p_\theta(x_i)^\gamma}
{\displaystyle \left\{\int p_\theta(x)^{\gamma+1}dx\right\}^{\gamma/(\gamma+1)}}.
\end{align}
The logarithmic empirical $\gamma$-loss used in this paper is the monotone
transformation
\begin{align}\label{gammaloss}
L_\gamma(\theta)
&=-\frac{1}{\gamma}\log R_\gamma(\theta)\\
&=-\frac{1}{\gamma}\log\left\{
\frac{1}{n}\sum_{i=1}^{n}p_\theta(x_i)^\gamma
\right\}
+\frac{1}{\gamma+1}\log\left\{
\int p_\theta(x)^{\gamma+1}dx\right\}.
\end{align}
Thus, minimizing $L_\gamma(\theta)$ is equivalent to maximizing
$R_\gamma(\theta)$, and either criterion yields the minimum
$\gamma$-divergence estimator.  The integral
$\int p_\theta^{\gamma+1}dx$ is intractable for many models with an unknown
normalizing constant, whereas the $\gamma$-Stein approach developed below
avoids this integral entirely.
}

\begin{example}
Let us consider a Normal model
\[
 p_\theta(x)=\frac{\exp\{-\frac12(x-\mu)\Sigma^{-1}(x-\mu)\}}{Z_\theta},
\]
where $\theta=(\mu,\Sigma^{-1})$.
Here the normalizing constant is known as  $Z_\theta=\det(2\pi\Sigma)^\frac12$.
{\color{black} Hence, the $\gamma$-loss \eqref{gammaloss} is written in closed form as
\[
       L_\gamma(\theta)=-\frac{1}{\gamma} 
\log\left[\frac1n\sum_{i=1}^n \exp\left\{-\frac{\gamma}{2}(x_i-\mu)\Sigma^{-1}(x_i-\mu)\right\}\right]
+\frac{1}{2}\frac{1}{\gamma+1}\log\{\det(2\pi\Sigma/(\gamma+1))\}
\]
}
This induces the estimating equation:
\[
     V_\gamma(\theta)=\frac1n\sum_{i=1}^n u_\theta(x_i)^\gamma
\begin{bmatrix}
\Sigma^{-1}(x_i-\mu)\\[2mm]
(x_i-\mu)(x_i-\mu)^\top-\frac{1}{\gamma+1}\Sigma\;
\end{bmatrix}=
\begin{bmatrix}
0\\[2mm]
\mathrm{O}
\end{bmatrix},
\]
where $u_\theta(x)=\exp\{-\frac{1}{2}(x-\mu)\Sigma^{-1}(x-\mu)\}$.
This form shows strong robustness by downweighting the contribution of observations $x_i$ with large Mahalanobis distance to the mean $\mu$, $(x_i-\mu)\Sigma^{-1}(x_i-\mu)$.
The minimum $\gamma$-divergence estimator for $\Sigma$ is deeply discussed in \citet{hung2022robust}.
Alternatively, the $\gamma$-Stein estimating equation is given by
\[
     U_\gamma(\theta)=\frac1n\sum_{i=1}^n u_\theta(x_i)^\gamma
\begin{bmatrix}
\Sigma^{-1}(x_i-\mu)\\[2mm]
({\gamma+1})\Sigma^{-1}(x_i\!-\!\mu)(x_i\!-\!\mu)^\top-\mathrm I_d\;
\end{bmatrix}=
\begin{bmatrix}
0\\[2mm]
\mathrm{O}
\end{bmatrix}
\]
 {where the second block is understood after symmetrization with respect to the precision-matrix parameter. Equivalently, multiplying the second equation by $\Sigma$ on the left gives the same covariance fixed point as the $\gamma$-divergence estimating equation.}
In this way, the estimating functions $V_\gamma(\theta)$ and $U_\gamma(\theta)$ are close to each other, however, $U_\gamma(\theta)$ is driven without any information of $Z_\theta$ and has asymmetric Jacobian matrix.
The fixed-point algorithms for solving $V_\gamma(\theta)=0$ and $U_\gamma(\theta)=0$ are equal as
\[
     \mu \leftarrow\; \frac{\sum_{i=1}^n u_\theta(x_i)^\gamma x_i}{\sum_{i=1}^n u_\theta(x_i)^\gamma}
\quad \text{ and } \quad
   \Sigma \leftarrow \;(\gamma+1) \frac{\sum_{i=1}^n u_\theta(x_i)^\gamma (x_i-\mu)(x_i-\mu)^\top}{\sum_{i=1}^n u_\theta(x_i)^\gamma}.
\]
\end{example}
In this way, the proposed estimator can be organized parallel to other established estimation procedures. 
We next discuss more advanced applications to some notable models.
The following examples are crucial applications of the geometric divergence and Laplacian operators defined on a unit sphere.

%\subsubsection*{Alternative validator: anchored $\gamma_0$-KSD.}

\subsection{von Mises--Fisher and Fisher-Bingham models}\label{sub32}

Consider a typical example on a unit sphere, in which the MLE and the $\gamma$-score matching estimator are both well organized for the parametric estimation.
On the compact sphere, there is no boundary, so the Stein identities apply without edge terms.
Let $x_1,\dots,x_n \in S^{d-1}\subset\mathbb{R}^d$ be unit vectors.
The von Mises--Fisher (vMF) density is
\[
p(x; \mu,\kappa) \;=\; C_d(\kappa)\,\exp\{\kappa\,\mu^\top x\},\qquad
\|\mu\|=1,\ \kappa\ge 0,
\]
with normalizing constant
\[
C_d(\kappa)\;=\;\frac{\kappa^{\nu}}{(2\pi)^{d/2}\,I_{\nu}(\kappa)},\qquad
\nu=\frac{d}{2}-1,
\]
where $I_\nu$ is the modified Bessel function of the first kind.
The MLE for $(\mu,\kappa)$ is given by solving
\[
(\mathrm I_d-\mu\mu^\top)\,R=0, \quad
\frac{I_{\nu+1}(\kappa)}{I_{\nu}(\kappa)}= \frac{R}{\|R\|},
\]
where  $R=\sum_{i=1}^n x_i$ and $ \nu=\frac{d}{2}-1$.

Let us look at geometric operators on $S^{d-1}$ for Stein identities.
For a smooth scalar $g:\,S^{d-1}\to\mathbb{R}$, with ambient extensions in $\mathbb{R}^d$,
\[
\nabla_S \;g(x) \;=\; (\mathrm I_d-xx^\top)\,\nabla_x g(x),
\]
\[
\Delta_S g(x) \;=\; \mathrm{tr}\!\big((\mathrm I_d-xx^\top)\,\nabla_x^2 g(x)\,(\mathrm I_d-xx^\top)\big) \;-\; (d-1)\,x^\top \nabla_x g(x).
\]
For the vMF model,
\[
\nabla_S \log p(x;\mu,\kappa) \;=\; \kappa\big(\mu-(\mu^\top x)\,x\big), \quad
\Delta_S \log p(x;\mu,\kappa) \;=\; -(d-1)\,\kappa\,(\mu^\top x).
\]
On the boundaryless manifold $S^{d-1}$, the Stein identity reads
\[
\mathbb{E}_{X\!\sim p}\!\left[\Delta_S g(X) + \langle \nabla_S g(X),\,\nabla_S \log p(X)\rangle\right] = 0.
\]
The $\gamma$-weighted version (used by $\gamma$-score matching) is
\[
\mathbb{E}_{X\!\sim p}\!\left[p(X)^{\gamma}\Big\{\Delta_S g(X) + (\gamma+1)\langle \nabla_S g(X),\,\nabla_S \log p(X)\rangle\Big\}\right] = 0.
\]
Thus, the $\gamma$-score matching estimating equation is given by
\[
(\mathrm I_d-\mu\mu^\top)\tilde R \;=\; 0
,\quad (d-1)\,m_1\;-\;\kappa\,\big(1-m_2\big)\;=\;0.
\]
where
\[
\tilde R=\frac{\sum_{i=1}^n w_i x_i}{\sum_{i=1}^n w_i},\quad m_1=\frac{\sum_{i=1}^n w_i \mu^\top x_i}{\sum_{i=1}^n w_i},\quad
m_2=\frac{\sum_{i=1}^n w_i (\mu^\top x_i)^2}{\sum_{i=1}^n w_i}
\]
with \(
w_i= \exp\{\gamma\,\kappa\,\mu^\top x_i\}.
\)
A practical fixed-point update (with weights held fixed within the step) is
\[
\begin{pmatrix}
\mu\\ \kappa
\end{pmatrix}
 \;\longleftarrow\; 
\begin{pmatrix}
\displaystyle \frac{\tilde R}{\|\tilde R\|}\\[5mm]
\displaystyle \frac{(d-1)\,m_1}{\,1-m_2\,}
\end{pmatrix}.
\]
In this way, the $\gamma$-score matching needs no knowledge of the normalizing constant $C_d(\kappa)$.

Strictly, misaligned means $|\mu^\top x|$ is small as $\gamma$ increases, yielding improved robustness under antipodal or orthogonal contamination at a modest efficiency cost under clean data.
We have a small simulation study with observations on the unit sphere $S^{d-1}$ with $d=3$. Clean samples
are drawn from a von Mises--Fisher distribution
\[
X \sim \mathrm{vMF}(\mu^*,\kappa^*),\qquad \mu^*=(1,0,0)^\top,\ \kappa^*=10.
\]
Sample size is fixed at $n=400$. To assess robustness we contaminate the data by
replacing an $\varepsilon$ fraction of the sample with an ``antipodal spike'':
\[
 {(1-\varepsilon) \mathrm{vMF}(\mu^*,\,\kappa^*)+\varepsilon \mathrm{vMF}(-\mu^*,\,50)}
\]
for $\varepsilon\in\{0,0.05,0.10,0.20\}$.
Each configuration is replicated $r=50$ times.

We compare the MLE with the $\gamma$-score matching
estimator. Both estimators are reported in the
$(\mu,\kappa)$ parameterization; for $\gamma$-score matching we choose $\gamma=0.0, 0.05, 0.1,0.2,0.3$.
The \emph{trace RMSE} of $\hat\mu$ is defined as
\[
\mathrm{RMSE}_{\mathrm{tr}}(\mu)
=\sqrt{\mathbb{E}\,\mathrm{tr}\!\big((\hat\mu\hat\mu^\top-\mu^*{\mu^*}^\top)^2\big)}
=\sqrt{\,2\,\mathbb{E}\big[1-(\hat\mu^\top \mu^*)^2\big]\,},
\]
and the integrated RMSE for $(\hat\kappa,\hat\mu)$ is given by the sum of the trace RMSE for $\hat\mu$ and RMSE for $\hat\kappa$ 
estimated across the $r=50$ replications. 

We give the integrated RMSE of $(\kappa,\mu)$ versus contamination.
The $\gamma$--score matching curve grows slowly with $\varepsilon$, while MLE degrades
under heavy contamination. In a representative run, at $\varepsilon=0.05$ the  RMSE for MLE is approximately $4.81$ compared to $0.56$ for $\gamma$-score matching of $\gamma=0.05$.
A concise summary  is given in Table~\ref{tab1}.

\begin{table}[htbp]
\caption{Integrated RMSEs for MLE vs. $\gamma$-Score Matching}
\label{tab1}
\centering
% Add \usepackage{booktabs} and \usepackage{multirow} to your document preamble
\begin{tabular}{l l r r r r}
\toprule
& & \multicolumn{4}{c}{Contamination Level ($\epsilon$)} \\
\cmidrule(l){3-6} 
\multicolumn{2}{l}{Estimator} & 0.00 & 0.05 & 0.10 & 0.20 \\
\midrule
\multicolumn{2}{l}{MLE} & {\bf 0.45} & 4.81 & 6.53 & 8.06 \\
\addlinespace % Adds a clean visual break
\multirow{5}{*}{$\gamma$-SME} & $\gamma=0.00$ &{\bf  0.45} & 0.88 & 1.66 & 3.50 \\
& $\gamma=0.05$ & 0.76 & {\bf 0.56} & {\bf 0.55} & 1.40 \\
& $\gamma=0.10$ & 1.29 & 1.19 & 1.08 & {\bf 0.73} \\
& $\gamma=0.20$ & 2.65 & 2.66 & 2.69 & 2.59 \\
& $\gamma=0.30$ & 4.45 & 4.49 & 4.56 & 4.44 \\
\bottomrule
\end{tabular}
\end{table}
This result suggests to build a data-adaptive selection for the robust parameter $\gamma$.
We discuss a method by $k$-fold cross validation (CV) with an anchored CV error in a subsequent discussion.

\subsubsection*{Fisher-Bingham (FB) model}
We discuss a natural and more complex extension of the von Mises-Fisher model,  highly flexible for modeling directional data but presents significant computational challenges.
The FB model on $S^{d-1}$ has density
\[
p(x ;\xi,B)\ =  \exp\{\xi^\top x + x^\top B x\},\qquad \|x\|=1,
\]
where $\xi\in\mathbb{R}^d$, $B$ is a symmetric matrix  with $\mathrm{tr}(B)=0$, and
\[
Z(\xi,B)=\int_{S^{d-1}}\exp\{\xi^\top x + x^\top B x\}\,d\sigma(x).
\]
Here $Z(\xi,B)$ is a hypergeometric function of a
matrix argument. Its stable evaluation (and derivatives) becomes computationally demanding as $d$ increases and/or $B$ is
anisotropic. This motivates normalizer-free estimation.
One observes
\begin{align}\nonumber
\nabla_S \log p(x ;\xi,B)&=(\mathrm I_d-xx^\top)\big(\xi + 2B x\big),
\\[2mm]\nonumber
\Delta_S \log p(x ;\xi,B)&= -(d-1)\,\xi^\top x\;-\;2d x^\top B x
\end{align}
since $\mathrm{tr}(B)=0$.
Thus, we observe  the $\gamma$-Stein identity on $S^{d-1}$:
For any smooth $g:S^{d-1}\to\mathbb{R}$,
\[
\mathbb{E}_{p}\!\left[p(X;\xi,B)^{\gamma}\Big\{\Delta_S \,g(X)+(\gamma+1)\,
\big\langle \nabla_S\, g(X),\,\nabla_S \log p(X;\xi,B)\big\rangle\Big\}\right]=0,
\]
where $p=p(\,\cdot\, ;\xi,B)$.
Let us take the canonical score function
\[
   \nabla_{\xi_i}\nabla_{S}\log p(x;\xi,B)=(\mathrm I_d-xx^\top)e_i
\]
and
\[ 
\nabla_{B_{jk}}\nabla_{S}\log p(x;\xi,B)=-2(\mathrm I_d-xx^\top)e_{j}e_{k}^\top x
\]
for $1\leq i,j,k\leq d$ as a set of test functions, where $e_i$ is the canonical orthonormal basis of $\mathbb R^d$.
Noting
\[
   \nabla_{\xi_i}\Delta_{S}\log p(x;\xi,B)=-(d-1)x_i\quad
\nabla_{B_{jk}}\Delta_{S}\log p(x;\xi,B)=-2 d x_j x_k,
\]
the $\gamma$ score matching estimating equation is given as follows:
\begin{align*}
\mathbb{E}_{p}\!\left[p(X;\xi,B)^{\gamma}\Big\{-(d-1)X_i+(\gamma+1)\,
\left( \xi_i + 2(BX)_i - X_i(\xi^\top X + 2X^\top B X) \right)\Big\}\right]=0,
\end{align*}
\begin{align*}
\mathbb{E}_{p}\!\left[p(X;\xi,B)^{\gamma}\Big\{2\delta_{jk} - 2dX_jX_k + (\gamma+1)\Big(X_j\xi_k + X_k\xi_j + 2(X_j(BX)_k + X_k(BX)_j) \right. \\ \left. - 2X_jX_k(\xi^\top X + 2X^\top B X)\Big)\Big\}\right]=0
\end{align*}
for all $i$  $(1\leq i\leq d)$ and all $j,k$ $(1 \leq j \leq k \leq d)$.
These equations, along with the constraint $\mathrm{tr}(B)=0$, can be solved numerically by replacing the expectation $\mathbb{E}_p[\,\cdot\,]$ with a sample average over observed data.
The procedure to solve the empirical equation for a given observations $x_1,\dots,x_n\in S^{d-1}$
 is computationally efficient, as it only involves matrix-vector products and solving small linear systems at each step. The robustness is inherited from the $\gamma$-weighting, which down-weights observations $x_i$ that are misaligned with the current estimate of $\xi$ or fall in directions penalized by $B$.

The application of the $\gamma$-score matching estimator to the von Mises-Fisher and Fisher-Bingham models gives its efficacy for distributions on the unit sphere. 
The key insight is that the method's foundation--the $\gamma$-Stein identity--can be readily adapted to any boundaryless manifold where appropriate surface gradient and Laplacian operators are defined.
This provides a clear path for extending the framework to other important distributions on classical manifolds used in multivariate analysis. For instance, the Stiefel manifold, which parameterizes sets of orthonormal frames, and the Grassmann manifold, which parameterizes subspaces, both host rich families of distributions for analyzing directional and frame data. Many of these models, such as the Bingham-Stiefel and matrix Langevin distributions, also suffer from computationally intractable normalizing constants. As detailed in \citet{chikuse2003statistics},
%Chikuse's "Statistics on Special Manifolds," 
these distributions are important in fields ranging from bioinformatics to computer vision. The normalizer-free and robust nature of the $\gamma$-score matching approach makes it a promising candidate for developing efficient inference procedures in these more complex geometric settings.

We next focus on a case where the $\gamma$-minimum divergence estimation is challenging because the empirical loss function involves intractable integral term.

\subsection{Normal Mixture Model (NMM)}\label{sub33}

Consider a normal mixture density modeled by
\begin{align*}
    p_\theta(x)=\sum_{j=1}^J \pi_j\,\phi_j(x; \mu_j,\Sigma_j),
\end{align*}
where $\theta=\{(\pi_j,\mu_j,\Sigma_j)\}_{j=1}^J$, and $\phi_j(x; \mu,\Sigma)$ is a normal density function with mean $\mu$ and variance $\Sigma$. 
The MLE is usually employed and satisfies the efficient computation via the EM algorithm.
However, the statistical performance is fragile under small misspecification.
The density power divergence method can be employed as a robust alternative.
For example, the minimum $\gamma$-divergence method introduces the empirical $\gamma$-loss function \eqref{gammaloss}.
It gives a robust estimator, however it involves an intractable integral unless $\gamma+1$ is a positive integer, see \citet{fujisawa2006robust} for the procedure with $\gamma$ fixed at $1$.
This brings inflexibility for selecting better estimators.
To mitigate this issue, we take the $\gamma$-Stein approach. 
The $\gamma$-score estimating function for the model $p_\theta(x)$ is given by
\begin{align*}
    U_\gamma(\theta,x) &= A^{(\gamma)}_{p_\theta}\big(\nabla_\theta s_\theta(x)\big) \\
    &= p_\theta(x)^\gamma\Big\{(\gamma+1)\,s_\theta(x)^\top \nabla_\theta s_\theta(x) + \nabla_x \cdot \nabla_\theta s_\theta(x)\Big\}
\end{align*}
with $s_\theta(x)=\nabla_x\log p_\theta(x)$.
 {Let $\Lambda_j=\Sigma_j^{-1}$ and let $s_j(x)=\nabla_x\log\phi_j(x;\mu_j,\Sigma_j)=\Lambda_j(\mu_j-x)$ denote the component score.}
However, the form is complicated, since $s_\theta$ is a responsibility-weighted
average of the component scores and its parameter derivatives involve the
derivatives of $r_j$ together with interactions among all components.
{\color{black}We select the following computationally tractable, component-wise test fields.  The notation $f^{(\Lambda_j)}[H]$ denotes the directional field associated with a symmetric perturbation $H$ of the precision matrix $\Lambda_j$.  These fields are not the parameter derivatives $\nabla_\theta s_\theta$; consequently, unlike the setting of Proposition~\ref{symmetry}, the resulting population Jacobian is generally a nonsymmetric cross-Gram matrix rather than a positive-definite Gram matrix:}
\[
 {
\begin{aligned}
f^{(\pi_j)}(x)      &:= s_j(x),\\[2pt]
f^{(\mu_j)}[v](x)     &:= r_j(x)\,v,\\[2pt]
f^{(\Lambda_j)}[H](x) &:= r_j(x)\,H\big(\mu_j-x\big),
\end{aligned}}
\qquad
r_j(x)=\frac{\pi_j\phi_j(x)}{p_\theta(x)},
\]
where $v\in\mathbb R^{d}$ and $H$ denote perturbation directions of $\mu_j$ and
of the precision matrix $\Lambda_j$, respectively.
These fields remain elementary---using only $r_j$ or $s_j$, with no Hessians or $\nabla_x s_\theta$, and their divergences are easy because
\[
\nabla_x r_j(x)=r_j(x)\{s_j(x)-s_\theta(x)\},\qquad
\nabla_x\!\cdot s_j(x)=-\operatorname{tr}(\Lambda_j).
\]
Resulting $\gamma$-Stein estimating functions (all explicit):
\[
\begin{aligned}
U_\gamma^{(\pi_j)}(\theta,x)
&=p_\theta(x)^\gamma\Bigl\{(\gamma+1)\langle s_\theta(x),s_j(x)\rangle
   -\operatorname{tr}\Lambda_j\Bigr\},\\[4pt]
U_\gamma^{(\mu_j)}(\theta,x)
&=p_\theta(x)^\gamma\,r_j(x)\bigl\{s_j(x)+\gamma\,s_\theta(x)\bigr\},\\[4pt]
U_\gamma^{(\Lambda_j)}(\theta,x)[H]
&=p_\theta(x)^\gamma\,r_j(x)\Bigl\{\bigl(s_j(x)+\gamma s_\theta(x)\bigr)^{\!\top}
   H\bigl(\mu_j-x\bigr)-\operatorname{tr}H\Bigr\}.
\end{aligned}
\]
Equivalently, for the precision block we may write the matrix form
\[
U_{\gamma,\text{mat}}^{(\Lambda_j)}(\theta,x)
=\;p_\theta(x)^\gamma\,r_j(x)\,\mathrm{sym}\!\Big((\mu_j-x)\big(s_j(x)+\gamma s_\theta(x)\big)^\top-\mathrm I\Big),
\]
so that $\langle U_{\gamma,\text{mat}}^{(\Lambda_j)},H\rangle_F=U_\gamma^{(\Lambda_j)}[H]$ for any symmetric $H$.
{\color{black}Detailed derivations are given in Appendix~\ref{Appendix-A}.
To state the relevant local-identification condition, let
$m_a(\theta)=\mathbb E_{p_{\theta_\star}}
[\mathcal A_{p_\theta}^{(\gamma)}f_a(\theta,X)]$, where $a$ indexes scalar
coordinates of the fields above, including vectorized matrix coordinates.
Differentiation at $\theta=\theta_\star$, together with the weighted Stein
identity, gives
\begin{align}\label{eq:mixture-cross-gram}
\frac{\partial m_a(\theta_\star)}{\partial\theta_b}
=\mathbb E_{p_{\theta_\star}}\!\left[
p_{\theta_\star}(X)^\gamma
\left\langle
\partial_{\theta_b}s_{\theta_\star}(X),
f_a(\theta_\star,X)
\right\rangle\right].
\end{align}
Hence the Jacobian is a cross-Gram matrix and need not be symmetric.
The mixing weights must also be treated on the simplex, because
$\sum_j\pi_j\partial_{\pi_j}s_\theta=0$ identically vanishes the radial
direction.  We therefore impose the explicit local rank condition that the
Jacobian in \eqref{eq:mixture-cross-gram}, after restriction to the tangent
space $\{\delta\pi:\sum_j\delta\pi_j=0\}$ and the remaining parameter
coordinates, is nonsingular at $\theta_\star$ (for a fixed labeling).}

 {To evaluate the practical performance and robustness of the proposed simplified
\(\gamma\)-Stein estimator, we conduct a simulation study comparing it with the
standard MLE, implemented via the expectation--maximization (EM) algorithm.
Following the reviewer's suggestion, we revise the previous symmetric mixture
setting and consider an asymmetric two-component, two-dimensional spherical
normal mixture model.  This design reduces the label-switching ambiguity and
makes the effect of mean-directed contamination more transparent.}

 {The clean data-generating distribution is}
{ 
\[
0.65\,N\{(-2,0)^\top,0.45I_2\}
+
0.35\,N\{(2,0)^\top,0.70I_2\}.
\]
}
 {For each experimental condition, we draw \(n=500\) observations.  To assess
robustness, we replace an \(\epsilon\) fraction of observations by outliers
generated from}
{ 
\[
N\{(6,0)^\top,0.25I_2\}.
\]
}
 {This contamination mechanism directly pulls the mean of the second component
to the right.  We consider
\(\epsilon\in\{0,0.05,0.10,0.15\}\) and repeat the experiment 100 times for each
level.  The component labels are aligned by the \(x\)-coordinate of the estimated
means before computing the errors.}

 {We first examined \(\gamma\in\{0.10,0.15,0.20,0.25\}\).  The results showed a
monotone robustness pattern: larger values of \(\gamma\) more strongly
down-weight the mean-directed contaminating cluster.  We therefore report
\(\gamma=0.25\) in Table~\ref{tab:gmm_rmse}, which gave the best overall
performance for the contaminated settings among the examined values.  The full
grid led to the same qualitative conclusion.}

 {Table~\ref{tab:gmm_rmse} shows the expected robustness--efficiency trade-off.
Under the clean model, MLE is more efficient, especially for the mixing weights.
However, once mean-directed contamination is introduced, the proposed estimator
substantially stabilizes the component mean and variance estimates.  At
\(\epsilon=0.05\), the RMSE for the component means decreases from \(0.248\)
for MLE to \(0.073\) for the \(\gamma\)-Stein estimator, and the error of the
second component mean decreases from \(0.493\) to \(0.135\).  At
\(\epsilon=0.10\), the corresponding errors are reduced from \(0.462\) to
\(0.272\) for the component means and from \(0.922\) to \(0.542\) for the
second component mean.  These results show that density-powered Stein weighting
effectively suppresses the influence of observations that pull the fitted mean
away from the main mixture component.}

 {For heavier contamination, such as \(\epsilon=0.15\), the problem becomes
intrinsically more difficult because the contaminating cluster starts to behave
like an additional mixture component.  Even in this case, the \(\gamma\)-Stein
estimator still improves the mean estimation, although the variance RMSE is no
longer uniformly better than that of MLE.  This illustrates that the proposed
method should be interpreted as a robust alternative with a clear
efficiency--robustness trade-off, rather than as a uniformly dominating
procedure.}

\begin{table}[htbp]
\centering
 
\caption{Monte Carlo performance for the asymmetric two-component normal
mixture under mean-directed contamination.  The clean mixture has weights
\((0.65,0.35)\), means \((-2,0)^\top\) and \((2,0)^\top\), and variances
\(0.45I_2\) and \(0.70I_2\).  An \(\epsilon\) fraction of observations is
replaced by draws from \(N\{(6,0)^\top,0.25I_2\}\).  Entries are averages over
100 replications.  The smaller value for each criterion at each contamination
level is bolded.}
\label{tab:gmm_rmse}
\bigskip
\begin{tabular}{@{}llrrrr@{}}
\toprule
\(\epsilon\)
& Estimator
& RMSE\((\pi)\)
& RMSE\((\mu)\)
& Error\((\mu_2)\)
& RMSE\((\sigma^2)\) \\
\midrule
0.00
& \(\gamma\)-Stein \((\gamma=0.25)\)
& 0.108 & 0.047 & 0.074 & 0.043 \\
& MLE (EM)
& \textbf{0.018} & \textbf{0.045} & \textbf{0.069} & \textbf{0.040} \\
\midrule
0.05
& \(\gamma\)-Stein \((\gamma=0.25)\)
& 0.089 & \textbf{0.073} & \textbf{0.135} & \textbf{0.178} \\
& MLE (EM)
& \textbf{0.038} & 0.248 & 0.493 & 0.625 \\
\midrule
0.10
& \(\gamma\)-Stein \((\gamma=0.25)\)
& 0.077 & \textbf{0.272} & \textbf{0.542} & \textbf{0.798} \\
& MLE (EM)
& \textbf{0.069} & 0.462 & 0.922 & 0.998 \\
\midrule
0.15
& \(\gamma\)-Stein \((\gamma=0.25)\)
& \textbf{0.081} & \textbf{0.545} & \textbf{1.089} & 1.318 \\
& MLE (EM)
& 0.101 & 0.659 & 1.317 & \textbf{1.187} \\
\bottomrule
\end{tabular}
\end{table}

\subsection{Quartic potential model}\label{sub34}

We next consider a quartic potential model to demonstrate the estimator's utility in a more complex and realistic scenario, cf.  \citet{kleinert2009path} for the meaning and roles in statistical mechanics.
The model is defined by the unnormalized density 
\begin{align}\nonumber
f_\theta(x) = \exp(\theta_1 x + \theta_2 x^2 + \theta_3 x^4).
\end{align}  
This model is an ideal test case because its normalizing constant is intractable, making standard MLE computationally demanding. The true parameters yield a bimodal distribution, as seen in Figure \ref{fig:quartic}.

We compare the $\gamma$-score matching estimators in two scenarios: one with clean data and one with outliers as given in Table \ref{tab:quartic_results}.

\begin{itemize}
\item {Scenario 1: No outliers.}
In an ideal, contamination-free setting, the MLE is, as expected, more statistically efficient and achieves a lower RMSE.
The $\gamma$-Stein estimators exhibit a slight loss of efficiency, which represents the classic trade-off between robustness and optimal performance on clean data. 

\item {Scenario 2: With outliers.} 
The difference between the MLE and the $\gamma$-Stein estimators becomes clearer
under contamination.
Conversely, the $\gamma$-Stein estimators remain remarkably stable.
The estimator with $\gamma= 0.3$, in particular, maintains an RMSE that is nearly identical to its performance on clean data, reducing the influence of the outliers. 
\end{itemize}
These estimates are obtained without the numerical integration required by the
MLE in this example.  Thus the experiment illustrates both the robustness effect
of the weighting and the normalizing-constant-free nature of the estimating
equation.

\begin{figure}[htbp]
\centering
\includegraphics[width=0.7\textwidth]{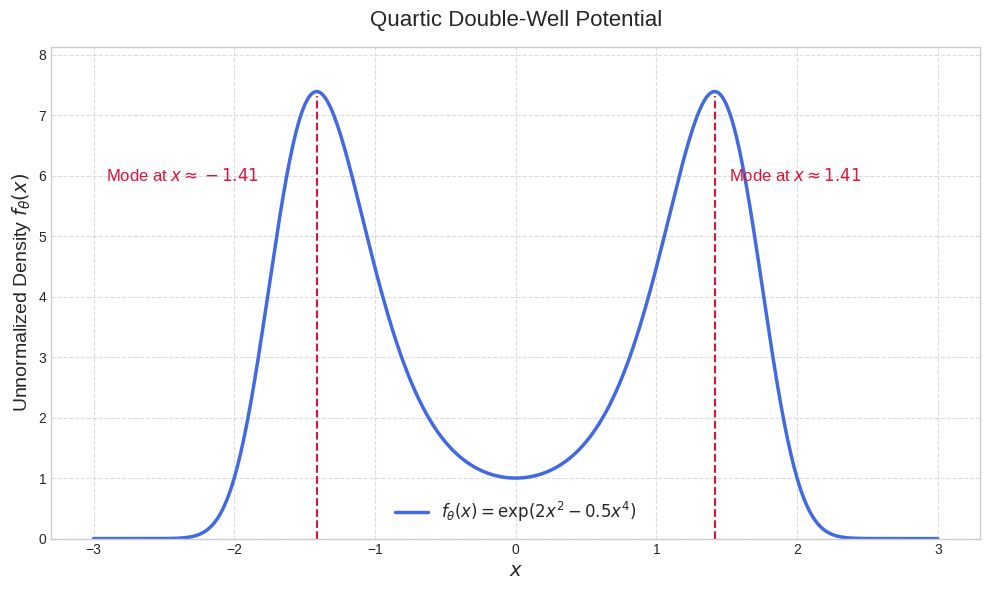} % User should provide the image
\caption{The bimodal shape of the unnormalized quartic potential density $f_\theta(x)$ for $\theta=(\,0, 2, -0.5)$.}
\label{fig:quartic}
\end{figure}

\begin{table}[htbp]
\centering
\caption{Performance comparison on the quartic potential model.}
\label{tab:quartic_results}
%{2mm}
\begin{tabular}{lrrrr}
\toprule
 {Estimator} &  {Mean $\hat{\theta}_1$} &  {Mean $\hat{\theta}_2$} &  {Mean $\hat{\theta}_3$} &  {RMSE} \\
\midrule
\multicolumn{5}{l}{\textit{Scenario: No Outliers}} \\
True & 0.0000 & 2.0000 & -0.5000 &  \\
MLE & -0.0043 & 2.0653 & -0.5204 & 0.2745 \\
$\gamma$-Stein ($\gamma=0.3$) & 0.0071 & 1.9241 & -0.4341 & 0.4128 \\
$\gamma$-Stein ($\gamma=0.5$) & -0.0778 & 2.9560 & -0.6490 & 1.4536 \\
\midrule
\multicolumn{5}{l}{\textit{Scenario: With Outliers}} \\
True & 0.0000 & 2.0000 & -0.5000 &  \\
MLE & 0.0416 & -0.0733 & -0.0215 & 2.1311 \\
$\gamma$-Stein ($\gamma=0.3$) & -0.0216 & 1.9387 & -0.4397 & 0.4542 \\
$\gamma$-Stein ($\gamma=0.5$) & 0.0475 & 2.3735 & -0.4968 & 0.5450 \\
\bottomrule
\end{tabular}
\end{table}

\subsection{Selecting the robustness parameter \texorpdfstring{$\gamma$}{gamma}}\label{sub35}

The robustness parameter $\gamma>0$ controls the bias-variance trade-off of the $\gamma$-Stein estimators: small $\gamma$ prioritizes efficiency under well-specified models, while larger $\gamma$ down-weights low-density (outlier) regions and improves robustness at a possible efficiency cost. We outline several principled, implementable strategies for choosing~$\gamma$.

\subsubsection*{Robust cross validation}
Split the data into $K$ folds with index sets $({\mathcal I_k})_{k=1}^K$.
Fix an \emph{anchor} level $\gamma_0\in(0,\gamma_{\max}]$ (e.g., $\gamma_0=0.1$) that defines the validation moment. 
For each candidate $\gamma$
\begin{enumerate}
\item {Fit:} Compute $\hat\theta_\gamma^{(-k)}$ on the training data ${\mathcal I_k}^c\,$ by solving the $\gamma$-Stein estimating equation.
\item {Validate:} Evaluate the anchored residual norm on the held-out fold:
\begin{align}\label{simple}
\mathrm{CV}_\gamma^{(k)}=\frac{1}{|\mathcal I_k|}\sum_{i\in\mathcal I_k}
U_{\gamma_0}\big(\hat\theta_\gamma^{(-k)},X_i\big)^\top
U_{\gamma_0}\big(\hat\theta_\gamma^{(-k)},X_i\big).
\end{align}
\end{enumerate}
Aggregate $\mathrm{CV}_\gamma=\frac{1}{K}\sum_{k=1}^K \mathrm{CV}_\gamma^{(k)}$, and find $\hat\gamma$ minimizing $\mathrm{CV}_\gamma$.
As a robust alternative to the squared $\gamma_0$-residual in \eqref{simple}, we may validate each fit $\hat\theta^{(-k)}_\gamma$ using the $\gamma_0$-kernelized Stein discrepancy (KSD), specialized to the model at hand. 
The general $\gamma$-KSD is formally introduced in Section \ref{KSD}; here we only use its empirical form for validation.
Replacing the squared $\gamma_0$-anchored  squared-residual with a $\gamma_0$-KSD gives a distribution-level, geometry-aware validation score rather than a moment-level proxy. 
It tends to (i) penalize shape mismatch more faithfully, (ii) be less sensitive to parametrization, and (iii) separate the anchor (robustness of the validator) from the $\gamma$-fitting cleanly.

The use of a fixed anchor `$\gamma_0$' in the validation step is a deliberate choice to stabilize the evaluation criterion. A more conventional approach might evaluate the residual norm using the same $\gamma$ that was used for fitting (i.e., using $U_{\gamma}(\hat\theta_\gamma^{(-k)},X_i)$). However, this would mean that the evaluation metric itself changes with each candidate $\gamma$, confounding the selection process. By fixing the validation moment at  $\gamma_0$, we ensure that all candidate models, parametrized by $\hat\theta_\gamma^{(-k)}$, are evaluated against the same, consistent benchmark. The choice of a small, positive $\gamma_0$ (e.g., $\gamma_0=0.1$) is motivated by the desire for a highly robust metric; it ensures that the validation score itself is resistant to outliers within the held-out fold $\mathcal{I}_k$. This design decouples the search for an optimal robustness-efficiency trade-off (governed by $\gamma$) from the need for a reliably robust evaluation framework provided by $\gamma_0$.
It could build an optimal weight matrix based on the theory of GMM, and define the anchored norms defined by the weight matrix.  
However, we choose the simple squared residual norm \eqref{simple} since our objective is to build robust selection for the tuning parameter $\gamma$ in the presence of outliers rather than to give more efficient estimator.
The performance of this method is illustrated below.

We return the simulation study for the vMF model  on the unit sphere $S^{2}$:
\[
X \sim \mathrm{vMF}(\mu^*,\kappa^*),\qquad \mu^*=(1,0,0)^\top,\ \kappa^*=10
\]
considering a $\varepsilon$-contamination model
\[
 {(1-\varepsilon) \mathrm{vMF}(\mu^*,\,\kappa^*)+\varepsilon \mathrm{vMF}(-\mu^*,\,50)}.
\]
Apply the method of robust cross validation using  the $\gamma_0$-KSD, which
will be discussed in the general formulation.
Concretely, let $\widehat{S}^2_{\mathrm{KSD},\gamma_0}(\hat\theta^{(-k)}_\gamma;\mathcal I_k)$ denote the unbiased U-statistic estimator of the squared $\gamma_0$-KSD on the held-out fold $\mathcal I_k$ (we rely only on its value up to a multiplicative constant). We then set
\[
\mathrm{CV}^{(k)}_{\gamma,\mathrm{KSD}}=\widehat{S}^2_{\mathrm{KSD},\gamma_0}(\hat\theta^{(-k)}_\gamma;\mathcal I_k),
\qquad
\mathrm{CV}_{\gamma,\mathrm{KSD}}=\frac{1}{K}\sum_{k=1}^K \mathrm{CV}^{(k)}_{\gamma,\mathrm{KSD}}.
\]
Our experiments report both the argmin and the ``one-SE’’ choice: among $\gamma$ whose mean CV is within one standard error of the minimum, we select the smallest $\gamma$. The general $\gamma$-KSD is formally introduced in Section \ref{sub41}; here we only use its empirical form for validation.
For this validation task, it can be understood as a robust, kernel-based tool that measures the distance between the model and the data. 
Using it as our validation score provides a more comprehensive, geometry-aware benchmark.

The selected $\gamma$ tracks the contamination level $\varepsilon$ in a stable, anchor-invariant manner: $\gamma\approx 0.05$ at $\varepsilon\approx 0.05$ and $\gamma\approx 0.10$ at $\varepsilon\in\{0.10,0.20\}$, see Table \ref{t2}.
On clean data ($\varepsilon=0$), the one-SE rule prefers a smaller $\gamma$ (near $0$), avoiding spurious over-robustness. We include a compact table of selections (mean CV and stability proportion across replications), and report that fixing the kernel bandwidth across folds further stabilizes the validator.
The results in Table \ref{t2} also suggest that the final selection of $\hat{\gamma}$ is robust to the choice of the anchor $\gamma_0$ itself, even for $\gamma_0=0$.

\begin{table}[!ht]
\centering
%\small
\caption{\protect {Performance of anchored $\gamma_0$-KSD cross-validation. The entry ``$\hat\gamma$/prop'' reports the selected value of $\gamma$ and the stability proportion, i.e. the proportion of replications in which that value was selected. The symbol ``--'' indicates that the clean case has no contamination-matching stability target.}}
\label{t2}
%*{3mm}
\bigskip
\begin{tabular}{l cc cc cc}
% --- START CORRECTION ---
\toprule % <-- Was \hline
& \multicolumn{2}{c}{$\gamma_0 = 0.00$} & \multicolumn{2}{c}{$\gamma_0 = 0.05$} & \multicolumn{2}{c}{$\gamma_0 = 0.10$} \\
\cmidrule(lr){2-3} \cmidrule(lr){4-5} \cmidrule(lr){6-7} % <-- Was \cline.
$\varepsilon$ & $\hat{\gamma}$ / prop & KSD & $\hat{\gamma}$ / prop & KSD & $\hat{\gamma}$ / prop & KSD \\
\midrule % <-- Was \hline
0.00 & 0.05 / --   & 2.3055 & 0.10 / --   & 2.4581 & 0.10 / --   & 2.6378 \\
0.05 & 0.05 / 0.88 & 1.8787 & 0.05 / 0.84 & 1.8565 & 0.05 / 0.80 & 1.8440 \\
0.10 & 1.00 / 1.00 & 1.8503 & 0.10 / 0.92 & 1.8687 & 0.10 / 0.80 & 1.8980 \\
0.20 & 0.10 / 0.32 & 1.8967 & 0.10 / 0.36 & 1.8787 & 0.10 / 0.36 & 1.8618 \\
\bottomrule % <-- Was \hline
% --- END CORRECTION ---
\end{tabular}

\medskip % <-- REMOVE THIS LINE
\hspace{5mm}{\small `` -- '' indicates no stability match}
\end{table}

% Table 1: Selected \hat{\gamma} (best-by-mean) and stability (prop)

\section{Further Stein inference methods}

We also define a $\gamma$-kernel Stein discrepancy and examine its behavior in a
contaminated goodness-of-fit setting.  Finally, we describe a corresponding
weighted version of Stein variational gradient descent, in which particles in
low target-density regions have reduced influence on the velocity field.

\subsection{\texorpdfstring{$\gamma$}{gamma}-kernelized goodness-of-fit}\label{KSD}\label{sub41}

To make the core idea of Stein's method practical for goodness-of-fit testing, One approach is to kernelize the Stein operator.
Instead of searching over an arbitrary space of test functions, this technique leverages the rich structure of a Reproducing Kernel Hilbert Space (RKHS). 
By selecting the test functions from the unit ball within an RKHS, the resulting discrepancy--known as the Kernel Stein Discrepancy (KSD)--can often be computed in a simple closed form using only the kernel function. This provides an elegant and practical measure of the difference between distributions. Crucially, if the chosen kernel is "characteristic," the KSD is zero if and only if the two distributions are identical, which guarantees that the resulting GoF test is consistent. This powerful combination of Stein's method and kernel spaces has become a cornerstone of modern non-parametric hypothesis testing,
see \citet{liu2016kernelized, chwialkowski2016kernel}.

 {A powerful application of the $\gamma$-Stein operator is the development of robust, non-parametric discrepancies for goodness-of-fit diagnostics. In this subsection, $p$ denotes the data-generating density and $q$ denotes the target model used in the Stein operator. For a classical uncontaminated GoF test, the null is $H_0:p=q$. Under contamination, however, the statistic is better interpreted as a robust diagnostic for the main component of the data-generating distribution rather than as a test that arbitrary outliers should be ignored without calibration. In the numerical experiment below we therefore calibrate the critical values under a contaminated null distribution. Rejection then means that, after allowing for the specified contamination mechanism, the main component departs from the target model.}

To do this, we move from the standard $L^2$ space to a more powerful RKHS, $\mathcal{H}_K$, which allows us to work with a rich class of test functions. This leads to the $\gamma$-Kernel Stein Discrepancy ($\gamma$-KSD).
The $\gamma$-KSD is defined as the maximum difference between the distributions $p$ and $q$ as measured by the $\gamma$-Stein operator over the unit ball of functions in the RKHS.
Let $K$ be a positive-definite kernel defining the RKHS $\mathcal{H}_K$. The squared $\gamma$-KSD between distributions $p$ and $q$ is given by:
\begin{align}\nonumber
S_{K}^{(\gamma)}(p\|q) = \sup_{f \in \mathcal{B}_K} \left( \mathbb{E}_{X \sim p}[\mathcal{A}_q^{(\gamma)}f(X)] \right)^2,
\end{align} 
where $\mathcal{B}_K$ is the unit ball in $\mathcal{H}_K^d$, see \citet{korba2021kernel}.

For this discrepancy to be the basis of a useful statistical test, it must satisfy two crucial properties:
\begin{itemize}
\item Characterization \& Test Consistency:
 For a test to be  consistent  (i.e., guaranteed to detect a true difference for large sample sizes), the discrepancy must be zero if and only if  the distributions are the same. This property holds if the kernel $K$ is  characteristic. 
For such kernels, $S_{K}^{(\gamma)}(p\|q) = 0 \iff q=p$. This ensures that a non-zero discrepancy is a true indicator of differing distributions.
\item   {Robustness: The operator's weighting term $q(x)^\gamma$ systematically down-weights observations in regions where the target model $q$ assigns low probability. This can improve the sensitivity of the statistic to departures in the bulk of the distribution when a small fraction of gross outliers is present. The price is that alternatives occurring mainly in low-density regions may be down-weighted, so the robust statistic and its calibration must be matched to the inferential goal.}
\end{itemize}
The power of the KSD framework is that it yields a closed-form expression that can be estimated from data.

\begin{theorem}{The $\gamma$-Kernel Stein Discrepancy has the closed-form expression}
\begin{align}\nonumber
S_{K}^{(\gamma)}(p\|q) = \mathbb{E}_{\,(X,X') \sim p^{\otimes 2}} \left[ q(X)^\gamma q(X')^\gamma u_{q,K}^{(\gamma)}(X,X') \right],
\end{align} 
where $u_{q,K}^{(\gamma)}$ is the Stein kernel:
\begin{align}\nonumber  
u_{q,K}^{(\gamma)}(x,x') = & (\gamma+1)^2s_q(x)^\top K(x,x')s_q(x') + (\gamma+1)s_q(x)^\top \nabla_{x'}K(x,x') \\[3mm]\label{concl}
 & + (\gamma+1)s_q(x')^\top \nabla_x K(x,x') + \mathrm{tr}\,(\nabla^2_{x,x'}K(x,x')).
\end{align}
\end{theorem}
\begin{proof}
We work in a vector-valued RKHS
\[
  \mathcal{H}_K^{d}
  =\bigl\{\,f=(f_1,\dots,f_d):f_j\in\mathcal{H}_K\bigr\},\qquad
  \langle f,g\rangle_{K}
  =\sum_{j=1}^{d}\langle f_j,g_j\rangle_{\mathcal{H}_K}.
\]
Fix the RKHS unit ball
\[
  \mathcal{B}_K
  =\bigl\{\,f\in\mathcal{H}_K^{d}:\lVert f\rVert_{K}\le 1\bigr\}.
\]
%The reproducing property gives, for each $f\in\mathcal{H}_k^{d}$,
%\[
%  f_j(x) =\langle f_j,k(x,\cdot)\rangle_{\mathcal{H}_k},\qquad
%  \partial_{x_j}f_j(x)
%  =\langle f_j,\partial_{x_j}k(x,\cdot)\rangle_{\mathcal{H}_k}.
% \]
Writing $K_x=K(x,\cdot)$ and $\nabla_x K_x$ for the gradient kernel vector,
\[
  \mathcal{A}_{q}^{(\gamma)}f(x)
  =\Bigl\langle
     f,\,
     {q(x)^{\gamma}}\bigl[(\gamma+1)\,s_q(x)K_x+\nabla_x K_x\bigr]
   \Bigr\rangle_{K}.
\]
Define the representer
\[
  g_{q}^{(\gamma)}(x)
  ={q(x)^{\gamma}}
     \Bigl[(\gamma+1)\,s_q(x)K_x+\nabla_x K_x\Bigr],
  \qquad
  \Psi=\mathbb{E}_{X\!\sim p}\bigl[g_{q}^{(\gamma)}(X)\bigr].
\]
Then,
$\mathbb{E}_{p}\!\bigl[\mathcal{A}_{q}^{(\gamma)}f\bigr]=\langle f,\Psi\rangle_{K}$, and hence, by Cauchy-Schwarz in $\mathcal{H}_K^{d}$,
\[
  S^{(\gamma)}_K(p,q)=\lVert\Psi\rVert_{K}.
\]
Squaring and expanding,
\[
  \Bigl(\sup_{f\in\mathcal{B}_K}
      \bigl|\mathbb{E}_{p}[\mathcal{A}_{q}^{(\gamma)}f]\bigr|\Bigr)^{2}
  =\langle \Psi,\Psi\rangle_{K}
  =\iint p(x)q(x)^{\gamma}\,p(x')q(x')^{\gamma}\,
        \langle g_{q}^{(\gamma)}(x),g_{q}^{(\gamma)}(x')\rangle_{K}\,dx\,dx'.
\]
Using the reproducing identities,
\begin{align}\nonumber
  \langle g_{q}^{(\gamma)}(x),g_{q}^{(\gamma)}(x')\rangle_{K}
  = &  
       (\gamma+1)^2 s_q(x)^{\top}K(x,x')s_q(x')
       +(\gamma+1)s_q(x)^{\top}\nabla_{x'}K(x,x')\\\nonumber
     &  +(\gamma+1)s_q(x')^{\top}\nabla_{x}K(x,x')
       +\operatorname{tr}\bigl(\nabla_{x,x'}^{2}K(x,x')\bigr),
\end{align}
which is $u_{q,K}^{(\gamma)}(x,x').$
This concludes \eqref{concl}.
\end{proof}

Given a dataset $\{x_i\}_{i=1}^n$ drawn from $p$, we can construct an unbiased U-statistic estimator for the squared discrepancy:
\begin{align}\nonumber
\hat{S}_{K,\gamma}^2 = \frac{1}{n(n-1)} \sum_{i \neq j} u(x_i)^\gamma u(x_j)^\gamma u_{q,K}^{(\gamma)}(x_i,x_j).
\end{align} 
This statistic can be computed without knowing the normalizing constant of $q$. Under $H_0$, its value will be close to zero; under $H_1$, under alternatives for which the discrepancy is nonzero, the statistic tends to take larger values.
By comparing $\hat{S}_{K,\gamma}^2$ to a critical value (obtained via bootstrap methods), we can perform the GoF test.

Theorem \ref{Stein} applies to the case where the test functions $f$ range over the unit ball of $L^{2}(\mu_{\gamma})$.
To kernelize the construction we replace that Banach space by a
\emph{reproducing kernel Hilbert space}
$(\mathcal{H}_{K},\langle\!\langle\cdot,\cdot\rangle\!\rangle_{K})$
with positive‑definite kernel
$K:\mathbb{R}^{d}\times\mathbb{R}^{d}\to\mathbb{R}$.
Because $\mathcal{H}_{K}\subset L^{2}_{\mathrm{loc}}$ whenever $K$ is
bounded, the $\gamma$‑Stein operator
\(
  \mathcal{A}_{q}^{(\gamma)}
\)
is well defined on vector‑valued members of $\mathcal{H}_{K}^{d}$.
 Setting $\gamma=0$ recovers the classical KSD.

%\subsubsection{Statistic and bootstrap calibration}\subsubsection{Numerical study}
 {To examine the finite-sample behavior of the proposed $\gamma$-KSD statistic, we conduct a simulation study for robust power. The aim is not to test a pure null distribution against arbitrary contamination, but to detect a small mean shift in the main data component when the same type of outlier contamination is present under both null and alternative hypotheses.}
\begin{itemize}
    \item   {Target main component:}  {The uncontaminated target model used in the Stein operator is $q(x)=\mathcal{N}(x\mid 0,\mathrm I_2)$.}
    \item   {Contaminated null hypothesis:}  {The observed data are generated from}
    { 
    \[
    P_0=(1-\varepsilon)\mathcal{N}(0,\mathrm I_2)+\varepsilon C,
    \]
    }
     {where $C=\mathcal{N}([5,5]^\top,\mathrm I_2)$ is a known simulation contaminating distribution and $\varepsilon=0.10$.}
    \item   {Contaminated alternative hypothesis:}  {The contaminating component is unchanged, but the main component is shifted:}
    { 
    \[
    P_\delta=(1-\varepsilon)\mathcal{N}([\delta,\delta]^\top,\mathrm I_2)+\varepsilon C.
    \]
    }
\end{itemize}
 {For this experiment, we set $n=200$ and the significance level $\alpha=0.05$. We estimate power for different shift magnitudes $\delta$ across 500 Monte Carlo replications. The critical value for each statistic is obtained by bootstrap calibration under the contaminated null $P_0$. Thus, at $\delta=0$ the rejection probability estimates the type-I error for the contaminated-null test. The empirical power of the standard KSD ($\gamma=0$) and the robust $\gamma$-KSD ($\gamma=0.3,0.5$) is presented in Table~\ref{tab:robust_power}.}

\begin{table}[htbp]
\centering
\caption{Estimated test power to detect a mean shift $\delta$ in the presence of 10\% contamination. The power for $\delta=0$ corresponds to the empirical Type I error rate.}
%*{3mm}
\label{tab:robust_power}
\bigskip
\begin{tabular}{@{}cccc@{}}
\toprule
 {Shift $\delta$} &  {KSD ($\gamma=0.0$)} &  {$\gamma$-KSD ($\gamma=0.3$)} &  {$\gamma$-KSD ($\gamma=0.5$)} \\
\midrule
0.00 & 0.052 & 0.048 & 0.054 \\
0.20 & 0.048 & 0.160 & 0.224 \\
0.40 & 0.058 & 0.552 & 0.710 \\
0.60 & 0.050 & 0.906 & 0.978 \\
0.80 & 0.044 & 0.998 & 1.000 \\
\bottomrule
\end{tabular}
\end{table}

 {The results show the behavior expected from the contaminated-null interpretation.}
\begin{itemize}
    \item   {Type-I error control: For $\delta=0$, all statistics are close to the nominal level $\alpha=0.05$, indicating that the bootstrap calibration under $P_0$ is adequate in this experiment.}
    \item   {Power under contaminated calibration: The standard KSD statistic has low power for this design because the outlying component contributes strongly to the unweighted Stein discrepancy and masks the small shift in the main component.}
    \item  {The $\gamma$-KSD statistics show increasing power as $\delta$ grows. The factor $q(x)^\gamma$ down-weights observations in the low target-density outlier region and makes the statistic more sensitive to the shift of the main data cloud.}
\end{itemize}
 {This experiment should therefore be read as evidence for robust detection of a bulk-distribution change under a specified contamination design, not as a claim that all outliers can be ignored in a classical GoF test.}
The anchored $\gamma_0$-KSD was already employed in Section \ref{sub35} as a robust cross-validation validator for selecting $\gamma$. The present section supplies the general definition and closed-form expression.
The test has a computational cost of $O(n^2)$, same as standard KSD, and requires no knowledge of the normalizing constant of $p$.
Under $H_0$, the U-statistic is degenerate and its distribution converges to an infinite weighted sum of $\chi^2$ random variables, $n\,\hat  S_{K,\gamma}^{2} \xrightarrow{d} \sum_{\ell=1}^{\infty}\lambda_\ell\,Z_\ell^{2}$. Critical values can be obtained via bootstrapping.
 {The weighting scheme $w(x)=q(x)^\gamma$ can improve robustness to outliers compared with standard KSD ($\gamma=0$), but it may also reduce sensitivity to alternatives concentrated in low target-density regions. This limitation is one reason why calibration and the choice of $\gamma$ should be tied to the intended diagnostic task.}

\subsection{\texorpdfstring{$\gamma$}{gamma}-variational inference}\label{sub42}
Many modern methods for variational inference aim to approximate a complex target probability density, $q(x)$, which may be difficult to sample from directly. Instead of finding an analytical form for an approximation, particle-based methods  use a set of $N$ points, or  particles  $\{x_i\}_{i=1}^N$, to represent the distribution.
More precisely, $q(x)$ is often assumed to be $u(x)/Z$ with a tractable unnormalized function $u$ and the intractable normalizing constant $Z$. 
The goal is to iteratively move these particles through the space so that their empirical distribution gradually transforms to match the target $q(x)$. Imagine the particles as a cloud of points; we want to "steer" this cloud until its shape matches the landscape of $q(x)$.
This process can be viewed as a  controlled diffusion. We define a  velocity field , $\phi(x)$, which is a function that tells each particle where to move next. The challenge is to find the optimal velocity field---the one that causes the particle cloud $\hat{p}_t$ to flow towards the target $q(x)$ as efficiently as possible. Efficiency is measured by the steepest descent on the  KL divergence, $D_{\mathrm {KL}}(p\|\hat{p}_t )$, where $\hat p_t$ denotes a kernel density estimate of the empirical particle measure.
This quantifies the distance between the two distributions.

SVGD constructs a velocity field that moves particles toward high-probability
regions of \(q\) while maintaining dispersion among particles.
The velocity field has two components:
\begin{align}\nonumber
x_i \leftarrow x_i + \varepsilon \phi^*(x_i),
\end{align} 
where $\varepsilon$ is a step size and $\phi^*(x_i)$ is the velocity at the particle's location. The brilliance of SVGD lies in its velocity field, which has two essential components:
\begin{align}\label{eq6}
\phi^*(x) = \mathbb{E}_{X \sim \hat{p}_t} \left[ K(X, x) s_q(X) + \nabla_X K(X, x) \right].
\end{align} 
Here $K(X, x) s_q(X)$ uses the score of the target density, $s_q(X) = \nabla_X \log q(X)$, which equals $\nabla_X \log u(X)$. This term pushes the particles in the direction of increasing log-probability, acting like a standard gradient ascent.
Alternatively, $\nabla_X K(X, x)$ uses the gradient of the kernel function, $K$. This term makes the particles interact and repel each other, preventing them from all collapsing to the same point and encouraging them to cover the full breadth of the target distribution.
While effective, standard SVGD can be sensitive to outliers or errant particles, as the score function $s_q(x)$ can be very large in the tails of the distribution, leading to unstable updates.

A straightforward way to make this process more robust is to introduce a  weighting scheme  into the velocity field, which can be achieved by leveraging the structure of the $\gamma$-Stein operator. We can define a modified, robust velocity field, $\phi^*_\gamma$, as:
\begin{align}\nonumber
\phi_{\gamma}^*(x) = \mathbb{E}_{X \sim \hat{p}_t} \left[ \mathcal{A}_q^{(\gamma)} K(X, x) \right].
\end{align} 
Spelled out, this becomes:
\begin{align}\nonumber
\phi_{\gamma}^*(x) = \frac{1}{N}\sum_{j=1}^{N} u(x_j)^{\gamma}\left[(\gamma+1)K(x_j,x)s_q(x_j)+\nabla_{x_j}K(x_j,x)\right],
\end{align} 
 {where the weight $q(x_j)^\gamma$ is replaced by $u(x_j)^\gamma$, absorbing the common factor $Z^{-\gamma}$ into the step size. The key modification is therefore the target-density weight $u(x_j)^\gamma$: particles lying in low target-density regions have a smaller contribution to the velocity field. This mirrors the classical SVGD update but replaces the KL descent field by a $\gamma$-weighted Stein field.}

 {If a particle $x_j$ is far in the tail of the target distribution, then $u(x_j)^\gamma$ is small for $\gamma>0$, and its influence on the update of the other particles is damped. This leads to a more stable flow under errant particles or contaminated posterior geometry.}
As the robustness parameter $\gamma \to 0$, these weights approach 1, and the method recovers the standard SVGD algorithm.
This general principle can also be applied to other related frameworks, such as evolution strategies, to create robust, gradient-free optimizers.

\subsubsection*{Numerical illustration}

We compare two transport targets for a Poisson log-linear regression with an intercept and $d=6$ standardized covariates.
Let $z_i=x_i^\top\alpha$ and $\mu_i=\exp(z_i)$.

\begin{enumerate}
\item {Standard SVGD ($\gamma=0$).}
The Bayesian posterior $p(\alpha\mid X,y)\propto u(\alpha)$ has unnormalized log-density
\[
\log u(\alpha)
=\sum_{i=1}^n\bigl(y_i\,z_i-\exp(z_i)\bigr)-\frac{1}{2s_0}\,\|\alpha\|^2 + C,
\]
and is approximated with the standard SVGD field \eqref{eq6} with $\gamma=0$.

\item {Robust $\gamma$-SVGD ($\gamma>0$).}
The target $\pi_\gamma(\alpha)$ is induced by the $\gamma$ divergence  loss
\[
L_\gamma(\alpha)
=-\frac{1}{n}\,\frac{1}{\gamma}\sum_{i=1}^n
\exp\!\Bigl\{\gamma\,y_i z_i-\frac{\gamma}{\gamma+1}\exp\!\bigl((\gamma+1)z_i\bigr)\Bigr\},
\]
see Section~2.7 in \citet{eguchi2024minimum} for the loss under the Poisson regression model.
We adopt a numerically stable \emph{log-sum-exp} surrogate for the likelihood part of $\log\pi_\gamma$
and use the corresponding $\gamma$-SVGD transport.  Transport weights over particles
use $\mathrm{softmax}(\gamma\log u)$ and are annealed from $0$ to the target~$\gamma$.
\end{enumerate}

\subsubsection*{Experimental design}
 {We simulate $n=400$ training pairs from a Poisson log-linear model and then contaminate either the covariates, the responses, or both. Covariate contamination at rate $\epsilon_x=0.10$ creates leverage points by replacing selected covariate rows with large outlying values. Outcome contamination at rate $\epsilon_y=0.10$ creates count spikes by inflating selected responses. The mixed scenario applies both mechanisms. A separate clean test set ($n=1500$) evaluates prediction, so the reported RMSE measures recovery of the underlying clean predictive mean rather than fit to the contaminated observations.}
We run SVGD with $M=32$ particles, $T=220$ iterations, RBF kernel (median heuristic, small jitter),
RMSProp preconditioning, step backtracking, and L2 projection.  To mitigate leverage,
we use a split normal distribution prior with larger variance on the intercept and stronger shrinkage on slopes.
We consider $\gamma\in\{0.00,0.02,0.05,0.08,0.10\}$.

\subsubsection*{Metrics and model selection}
We report posterior-predictive \emph{RMSE} of $\hat\mu$ on the clean test set (lower is better),
as mean~$\pm$ standard error over $R$ replicates.
Let $\{\alpha^{(m)}\}_{m=1}^M$ denote the SVGD particles for the clean test set  $\{(x_i^\ast,y_i^\ast)\}_{i=1}^{n_\ast}$.
We then average across particles to obtain the posterior predictive mean
\[
\hat\mu_i=\frac{1}{M}\sum_{m=1}^M \mu_i^{(m)},
\]
where \(\mu_i^{(m)}=\exp( x_i^{\ast\top}\alpha^{(m)})\).
For a fixed robustness level $\gamma$, 
over $R$ independent replicates we report the mean and standard error for the root mean square error (RMSE),
\[
\overline{\mathrm{RMSE}}_\gamma=\frac{1}{R}\sum_{r=1}^R\mathrm{RMSE}_\gamma^{(r)},
\qquad
\mathrm{SE}_\gamma=\frac{\mathrm{sd}\bigl(\mathrm{RMSE}_\gamma^{(1)},\ldots,\mathrm{RMSE}_\gamma^{(R)}\bigr)}{\sqrt{R}},
\]
where $\mathrm{RMSE}_\gamma^{(r)}$ is RMSE on the test set, $\{\mu_i^{(r)},y_i^{*(r)}\}$.
The robustness level is chosen by the \emph{one-SE rule}:
among $\gamma$ whose mean RMSE is within one standard error of the empirical minimum,
we select the smallest~$\gamma$.

\subsubsection*{Results}
Table~\ref{tab:rmse} summarizes the RMSE (mean~$\pm$~s.e.) across scenarios
(clean; $Y$-contamination; $X$-contamination; mixed $X{+}Y$).
Figure~\ref{fig:rmse_panels} shows RMSE versus~$\gamma$ with error bars.
In brief:  {(i)} under clean data, $\gamma=0$ is optimal (no robustness tax);
 {(ii)} under outcome contamination, a moderate robustness level $\gamma\approx0.10$ yields the best accuracy;
 {(iii)} under covariate or mixed contamination, a light robustness level $\gamma\approx0.02$ performs best.
These findings support the use of a small default robustness ($\gamma\simeq0.02$),
escalated to $\gamma\simeq0.10$ when heavy right-tail anomalies in counts are suspected.
 {The fact that the selected $\gamma$ is smaller in the mixed $X+Y$ contamination case than in the pure $Y$-contamination case is not paradoxical. Response spikes mainly create tail observations in the likelihood contribution, for which stronger density weighting is beneficial. Leverage contamination in $X$, however, also changes the geometry of the score and can make an overly large $\gamma$ concentrate the particle flow too strongly around high-density regions, leading to under-adaptation for prediction. A light positive value such as $\gamma=0.02$ gives enough damping to stabilize the flow while avoiding excessive loss of information.}

% ---- RMSE table (paste your generated numbers) ----
%\vspace{-3mm}
\begin{table}[!h]
\centering
\caption{Posterior predictive RMSE (mean $\pm$ s.e.) over $R$ replicates.
Bold indicates the one-SE rule selection in each scenario (ties broken by smaller $\gamma$).}
\label{tab:rmse}
%\vspace{1mm}
\bigskip
\begin{tabular}{lcccc}
\toprule
$\gamma$ & clean & Y-contam & X-contam & X+Y-contam \\
\midrule
0.00 & \textbf{${\bf18.567} \pm 1.914$} & $27.159 \pm 4.283$ & $27.693 \pm 4.024$ & $32.321 \pm 4.068$ \\
0.02 & $23.649 \pm 3.075$ & $21.637 \pm 1.737$ & {${\bf21.306} \pm 2.089$} & {${\bf20.179} \pm 2.334$} \\
0.05 & $28.691 \pm 3.746$ & $21.368 \pm 2.153$ & $23.490 \pm 2.427$ & $24.558 \pm 2.662$ \\
0.08 & $29.739 \pm 2.768$ & $28.219 \pm 3.089$ & $27.225 \pm 2.834$ & $27.612 \pm 3.348$ \\
0.10 & $23.228 \pm 3.660$ & ${\bf18.341} \pm 2.215$ & $25.872 \pm 3.573$ & $24.405 \pm 3.399$ \\
\bottomrule
\end{tabular}
\end{table}

%\vspace{-3mm}
\begin{figure}[H]
\centering
\includegraphics[width=0.90\linewidth]{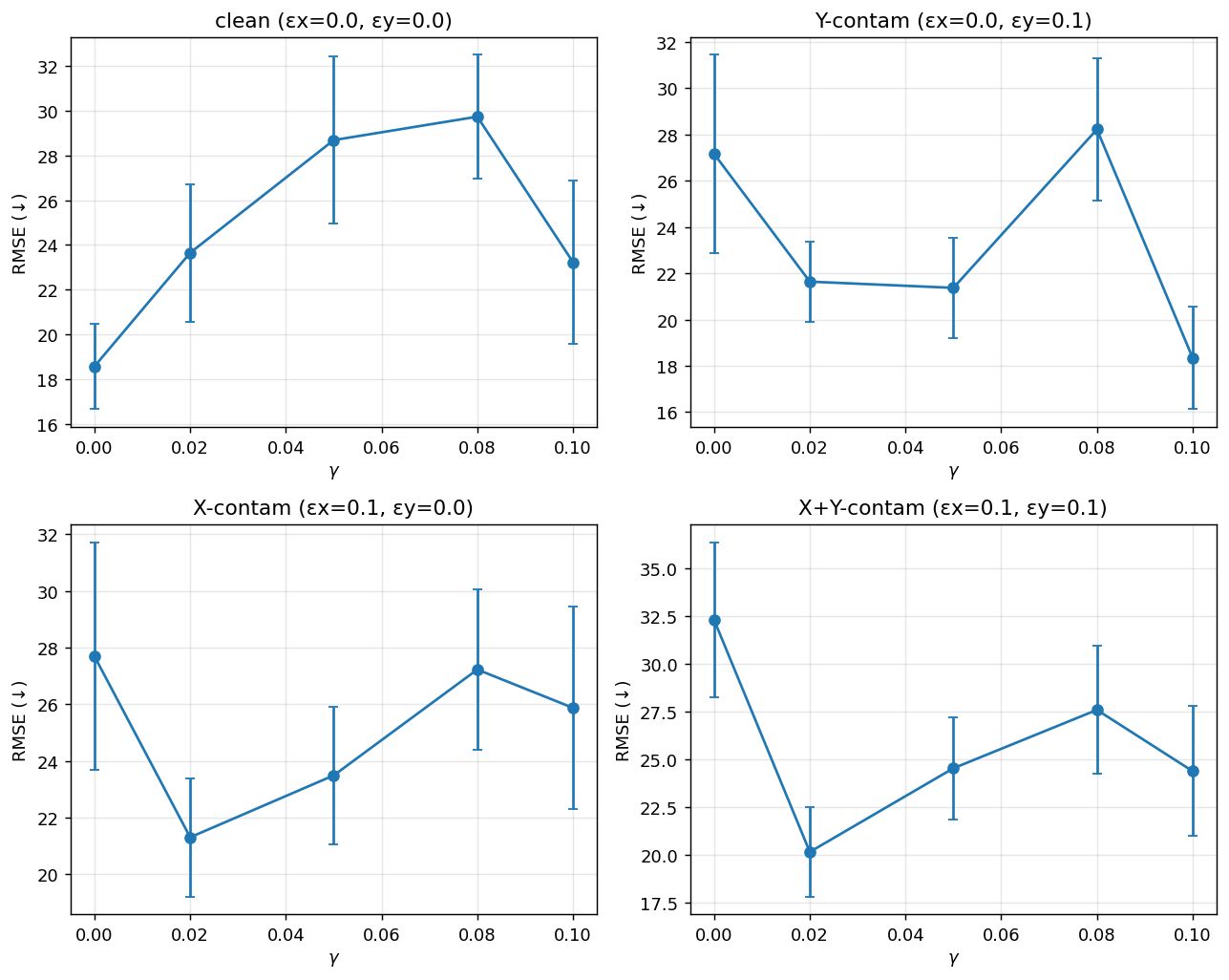}
\caption{Posterior-predictive RMSE versus $\gamma$ under four scenarios
(clean; $Y$-contamination; $X$-contamination; mixed $X{+}Y$).
Error bars show $\pm$ one standard error over replicates.}
\label{fig:rmse_panels}
\end{figure}

\section{Discussion}

We have formulated a density-power weighted version of the Stein operator.
The resulting operator weights the classical Stein field by \(q^\gamma\), which
leads to estimating equations and discrepancies that can be evaluated for
unnormalized models.  The transport-variation identity relates this construction
to the first variation of the \(\gamma\)-divergence.

Classical Stein operators perform well when data and model are well-aligned, but they can be brittle under contamination. The $\gamma$-weight introduces a controlled insensitivity to low-density regions: inliers continue to shape the fit, while outliers exert a much weaker influence. Conceptually, the method combines two forces already present in Stein flows--the ascent along the score and the repulsive spreading--then modulates both by $q^\gamma$. The result is a flow that concentrates learning effort where the model believes the signal lives.

 {A naïve weighted Fisher objective would involve the unknown 
$s_p$, making it impractical.  The variational view clarifies the origin of 
the operator: for the logarithmic $\gamma$-divergence, 
Proposition~\ref{gamma-div} expresses
\(\mathcal A_q^{(\gamma)}\) directly as the first variation of the
logarithmic \(\gamma\)-divergence along the escort transport.
Geometrically, $\mathcal{A}_q^{(\gamma)}$ is therefore the 
canonical Stein operator for the $(\gamma+1)$-escort coordinate of 
$D_\gamma$, not a deformation of the classical Stein operator.  The 
estimating equations themselves rely on the zero-expectation Stein 
identity, which is why they can be implemented without knowing the 
normalizing constant.  Thus, the divergence calculus motivates the operator 
through the escort transport, while the Stein identity supplies the usable 
estimating equation.}

The $\gamma$-Stein method has the following performance:
\begin{itemize}
\item \emph{Unnormalized models.} Estimating equations depend only on $\nabla_x\log u_\theta$ and the weight $u_\theta^\gamma$; the partition function cancels. 
 {This is useful for energy-based models, random field models, and other settings
where likelihoods are expensive or intractable.
\item \emph{Tuning $\gamma$.} Small positive values can provide a useful robustness--efficiency compromise, but the appropriate range depends on the model, the contamination mechanism, and the inferential task. Larger values emphasize outlier resistance at the cost of variance and possible loss of sensitivity to low-density alternatives.}
\item \emph{Algorithms.} Replacing the standard Stein field with its $\gamma$-weighted version yields robust particle methods (e.g., $\gamma$-SVGD) and robust discrepancies (e.g., $\gamma$-KSD) without changing the surrounding optimization scaffolding.
\end{itemize}

Let us overview a relation to existing robustness tools.
The method is philosophically close to density-power approaches that temper the likelihood. The difference is structural: here, robustness appears at the level of the Stein operator and its induced flow, tied to a transport derivative of a divergence. This yields (i) a direct route to score-matching-type estimators, (ii) natural compatibility with unnormalized models, and (iii) operator identities that extend to kernelized discrepancies and particle methods.

The $\gamma$-Stein machinery is most useful when one expects a small fraction of gross errors or heavy tails and wishes to preserve the convenience of score-based learning. When contamination is negligible and the model is nearly correct, $\gamma=0$ recovers the familiar Fisher/Stein landscape and is statistically most efficient. In high-noise, high-dimension regimes, modest $\gamma>0$ can stabilize estimation and improve out-of-sample behavior.
On the other hand, the $\gamma$-Stein method has the following limitations as a statistical procedure.
First, robustness trades efficiency: if $\gamma$ is too large, variance inflates and modes with low model mass may be under-explored. Second, the mixed measure $\mu_\gamma=p\,q^\gamma dx$ couples data and model in ways that complicate analysis under severe misspecification (e.g., overly diffuse $q$). Third, kernel and feature choices in $\gamma$-KSD and particle implementations remain important in high dimensions. These limitations point to the need for principled tuning and adaptivity.

Here are directions for a further development for the $\gamma$-Stein approach:
\begin{itemize}
\item \emph{Adaptive weighting.} Data- or iteration-dependent $\gamma$ (or spatially varying $\gamma(x)$) that remains scale-invariant for unnormalized targets.
\item \emph{General weights $w(q)$.} Beyond power laws, which weights preserve key invariances and yield tractable calculus? The scale-invariance argument narrows the field, but structured relaxations may be possible.
\item \emph{Theory under misspecification.} Non-asymptotic guarantees for $\gamma$-KSD testing and rates for $\gamma$-score matching with heavy tails or leverage points.
\item \emph{Manifold and discrete spaces.} Extending $\gamma$-Stein identities to Riemannian settings and to discrete models where IBP is replaced by summation-by-parts operators.
\item \emph{Applications.} Robust training of energy-based deep models, stable posterior transport in variational inference, and scientific domains where outliers or heavy-tailed observations are common (e.g., ecology, genomics, remote sensing).
\end{itemize}

In summary, the $\gamma$-Stein operator gives an operator-level representation
of density-power weighting in Stein-based inference.  This viewpoint connects
robust weighting, transport variation, and score-based computation, while
retaining the advantage that the normalizing constant is not required.

\section*{Data and code availability}
 {The Python notebooks used to reproduce the simulation studies  are available at \url{https://github.com/shinto-eguchi/Gamma-Stein}.}

\section*{Acknowledgements}
 {
The author thanks the editor and the reviewers for their careful reading and
constructive comments, which substantially improved the presentation of the
paper.  The author also thanks colleagues for helpful discussions on robust
divergence-based inference and Stein-type methods.
}

\bibliographystyle{plainnat}
\bibliography{references}

\newpage
\section*{Appendix}
%\subsubsection*{A1: Proof for the $\gamma$-divergence}

\appendix
\subsection*{Estimating functions for the normal mixture model}\label{Appendix-A}
\label{app:nmm}

{\color{black}
Throughout this appendix we write
\[
p_\theta(x)=\sum_{j=1}^{J}\pi_j\phi_j(x),\qquad
\phi_j(x)=\phi(x;\mu_j,\Sigma_j),\qquad
\Lambda_j=\Sigma_j^{-1},
\]
\[
s_j(x)=\nabla_x\log\phi_j(x)=\Lambda_j(\mu_j-x),\qquad
r_j(x)=\frac{\pi_j\phi_j(x)}{p_\theta(x)},\qquad
s_\theta(x)=\sum_{k=1}^{J}r_k(x)s_k(x),
\]
and we assume the boundary conditions under which
$\int\nabla_x\!\cdot\!\{p_\theta(x)^{\gamma+1}f(x)\}dx=0$ holds for the test
fields used below; this is satisfied because every field considered here grows
at most polynomially while $p_\theta^{\gamma+1}$ decays at a Gaussian rate.

Since $\log r_j=\log\pi_j+\log\phi_j-\log p_\theta$, we have
\begin{align}\label{a:gradr}
\nabla_x r_j(x)=r_j(x)\bigl\{s_j(x)-s_\theta(x)\bigr\},
\end{align}
and, because $s_j$ is affine in $x$,
\begin{align}\label{a:divs}
\nabla_x\!\cdot s_j(x)=-\operatorname{tr}\Lambda_j,
\qquad
\nabla_x\!\cdot\bigl\{H(\mu_j-x)\bigr\}=-\operatorname{tr}H,
\qquad
\nabla_x\!\cdot v=0
\end{align}
for a constant matrix $H$ and a constant vector $v$.

\begin{lemma}[component-tilted representation]\label{lem:master}
For any differentiable field $g:\mathbb R^{d}\to\mathbb R^{d}$, let
$\widetilde p_j(x)=\pi_j\phi_j(x)p_\theta(x)^\gamma$.  Then
\begin{align}\notag
\mathcal A^{(\gamma)}_{p_\theta}\bigl(r_j\,g\bigr)(x)
&=\frac{\nabla_x\!\cdot\!\{\widetilde p_j(x)g(x)\}}{p_\theta(x)}\\[2mm]
&=p_\theta(x)^\gamma r_j(x)\,
\mathcal A_{\widetilde p_j}g(x)\\[2mm]
&=p_\theta(x)^{\gamma}r_j(x)
\Bigl\{\bigl\langle s_j(x)+\gamma s_\theta(x),\,g(x)\bigr\rangle
+\nabla_x\!\cdot g(x)\Bigr\}.
\label{a:master}
\end{align}
Here $\mathcal A_{\widetilde p_j}$ is the ordinary Stein operator associated
with the possibly unnormalized tilted density $\widetilde p_j$.
\end{lemma}
\begin{proof}
By definition
$\mathcal A^{(\gamma)}_{p_\theta}f
=p_\theta^{\gamma}\{(\gamma+1)\langle s_\theta,f\rangle+\nabla_x\!\cdot f\}$.
Putting $f=r_jg$ and using the product rule together with \eqref{a:gradr},
\[
\nabla_x\!\cdot(r_jg)=\langle\nabla_x r_j,g\rangle+r_j\nabla_x\!\cdot g
=r_j\bigl\{\langle s_j-s_\theta,g\rangle+\nabla_x\!\cdot g\bigr\},
\]
so that
\[
\mathcal A^{(\gamma)}_{p_\theta}(r_jg)
=p_\theta^{\gamma}r_j\bigl\{\langle(\gamma+1)s_\theta+s_j-s_\theta,\,g\rangle
+\nabla_x\!\cdot g\bigr\}
=p_\theta^{\gamma}r_j\bigl\{\langle s_j+\gamma s_\theta,g\rangle
+\nabla_x\!\cdot g\bigr\}.
\]
The first equality in \eqref{a:master} follows from
$s_j+\gamma s_\theta=\nabla_x\log(\pi_j\phi_j p_\theta^{\gamma})$ and
$p_\theta^{\gamma}r_j=\pi_j\phi_j p_\theta^{\gamma}/p_\theta$.
\end{proof}

Lemma~\ref{lem:master} says that, for fields of the form $r_jg$, the
$\gamma$-Stein operator of the mixture equals the ordinary Stein operator of
the tilted density $\widetilde p_j=\pi_j\phi_jp_\theta^\gamma$, multiplied by
the density ratio $\widetilde p_j/p_\theta=p_\theta^\gamma r_j$.  In
particular, the unbiasedness
$\mathbb E_{p_\theta}[\mathcal A^{(\gamma)}_{p_\theta}(r_jg)]
=\int\nabla_x\!\cdot(\pi_j\phi_jp_\theta^{\gamma}g)\,dx=0$
is immediate.  The mean and precision blocks below follow by choosing $g$; the mixing-proportion block is obtained directly from the definition.

\paragraph{Mixing proportions.}
The field is $f^{(\pi_j)}=s_j$, which does not carry the factor $r_j$, so we use
the definition directly.  By \eqref{a:divs},
\[
U^{(\pi_j)}_\gamma(\theta,x)
=\mathcal A^{(\gamma)}_{p_\theta}(s_j)(x)
=p_\theta(x)^{\gamma}\Bigl\{(\gamma+1)\langle s_\theta(x),s_j(x)\rangle
-\operatorname{tr}\Lambda_j\Bigr\}.
\]

\paragraph{Component means.}
Let $v\in\mathbb R^{d}$ be a perturbation direction of $\mu_j$ and take
$g\equiv v$.  Then $\nabla_x\!\cdot v=0$ and Lemma \ref{lem:master} gives
\[
U^{(\mu_j)}_\gamma(\theta,x)[v]
=p_\theta(x)^{\gamma}r_j(x)\bigl\langle s_j(x)+\gamma s_\theta(x),\,v\bigr\rangle ,
\]
that is, in vector form,
\[
U^{(\mu_j)}_\gamma(\theta,x)
=p_\theta(x)^{\gamma}r_j(x)\bigl\{s_j(x)+\gamma s_\theta(x)\bigr\}.
\]

\paragraph{Precision matrices.}
Let $H$ be a symmetric perturbation direction of $\Lambda_j$ and take
$g(x)=H(\mu_j-x)$, so that $\nabla_x\!\cdot g=-\operatorname{tr}H$ by
\eqref{a:divs}.  Lemma \ref{lem:master} gives
\[
U^{(\Lambda_j)}_\gamma(\theta,x)[H]
=p_\theta(x)^{\gamma}r_j(x)
\Bigl\{\bigl(s_j(x)+\gamma s_\theta(x)\bigr)^{\!\top}H(\mu_j-x)
-\operatorname{tr}H\Bigr\}.
\]
Because $\langle A,H\rangle_F=\operatorname{tr}(A^\top H)$ and
$\operatorname{tr}\bigl\{(s_j+\gamma s_\theta)(\mu_j-x)^\top H\bigr\}
=(s_j+\gamma s_\theta)^\top H(\mu_j-x)$ for symmetric $H$, this linear
functional is represented by the matrix
\[
U^{(\Lambda_j)}_{\gamma,\mathrm{mat}}(\theta,x)
=p_\theta(x)^{\gamma}r_j(x)\,
\mathrm{sym}\Bigl\{(\mu_j-x)\bigl(s_j(x)+\gamma s_\theta(x)\bigr)^{\!\top}
-\mathrm I_d\Bigr\},
\]
in the sense that
$\langle U^{(\Lambda_j)}_{\gamma,\mathrm{mat}},H\rangle_F
=U^{(\Lambda_j)}_\gamma[H]$ for every symmetric $H$.
}

%\subsubsection*{A2: Proof for the $\beta$-divergence}

\end{document}